\documentclass[11pt]{article}

\newcommand{\arxivbuild}{}

\usepackage[letterpaper, margin=1.1in]{geometry}
\usepackage[T1]{fontenc}
\usepackage[utf8]{inputenc}
\usepackage{lmodern}
\usepackage{inconsolata}
\usepackage{microtype}
\usepackage[skip=0.55\baselineskip plus 2pt]{parskip}

\newcommand{\panelcodefont}{\fontsize{9pt}{11pt}\selectfont}

\usepackage[font=small,labelfont=bf,skip=6pt]{caption}

\usepackage{float}
\usepackage{placeins}   %
\usepackage{needspace}
\renewenvironment{figure*}[1][tbp]{\begin{figure}[#1]}{\end{figure}}
\renewenvironment{table*}[1][tbp]{\begin{table}[#1]}{\end{table}}

\usepackage[round,sort]{natbib}
\usepackage{graphicx}
\usepackage{bm}
\usepackage{multirow}
\usepackage{booktabs}
\usepackage{axlib}      %

\usepackage{thmtools}
\usepackage{thm-restate}
\makeatletter
\newcommand{\restatelabel}[1]{\ifthmt@thisistheone\label{#1}\else\tag{\ref*{#1}}\fi}
\makeatother

\definecolor{stdnavy}{HTML}{1A4F9C}
\hypersetup{colorlinks=true, linkcolor=stdnavy, citecolor=stdnavy, urlcolor=stdnavy,
            pdftitle={Multi-Target Few-Shot Prompting with Diffusion Language Models},
            pdfauthor={Anton Xue, Alan Li, Rahul Saha}}

\usepackage{authblk}

\title{\fontsize{16}{19}\selectfont Multi-Target Few-Shot Prompting with Diffusion Language Models}

\author{Anton Xue}
\author{Alan Li}
\author{Rahul Saha}
\affil{University of Texas at Austin}
\affil{\ttfamily
  \href{mailto:anton.xue@utexas.edu}{anton.xue@utexas.edu}\quad
  \href{mailto:alan.li@cs.utexas.edu}{alan.li@cs.utexas.edu}\quad
  \href{mailto:rahul.saha@utexas.edu}{rahul.saha@utexas.edu}}
\newcommand{\artifactlogo}[3]{%
  \raisebox{-#1\height}{\includegraphics[height=#2]{#3}}%
}
\newcommand{\codelinkrow}{%
  \artifactlogo{0.3}{1.1em}{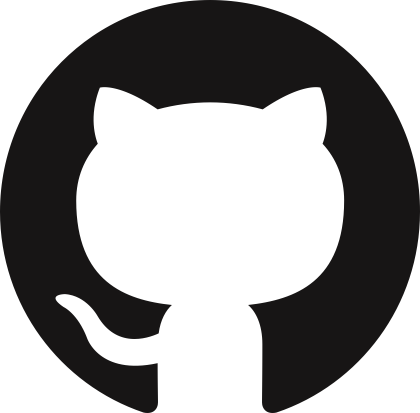}\hspace{0.42em}%
  \href{https://github.com/AntonXue/fewshot-dlms}{\texttt{github.com/AntonXue/fewshot-dlms}}%
}

\affil{\vspace{0.4em}\protect\codelinkrow}
\date{}

\begin{document}

\maketitle

\vspace{\stretch{0.5}}
\begin{abstract}
Few-shot prompting is useful when instructions and rewards are hard to specify.
One such application is editing by example, where a demonstrated formatting change or API migration may need to be applied across multiple sites.
However, such edits often depend on both left and right context, making generation with causal autoregressive models awkward.
We therefore introduce multi-target few-shot prompting with diffusion language models, where examples guide the infilling of multiple masked regions within a shared canvas.
To address the known challenge of reliably following demonstrations in few-shot prompting, we further leverage selective masking of the canvas to derive a likelihood-ratio reward that favors example-supported infills.
Across few-shot prompting benchmarks and software refactoring case studies, we find that demonstrations improve infilling accuracy, with example guidance yielding further gains in several settings.
\end{abstract}

\vspace{\stretch{0.8}}
\ifdefined\arxivbuild\begin{figure*}[!h]\else\begin{figure*}[t]\fi
    \centering
    \includegraphics[width=0.95\textwidth]{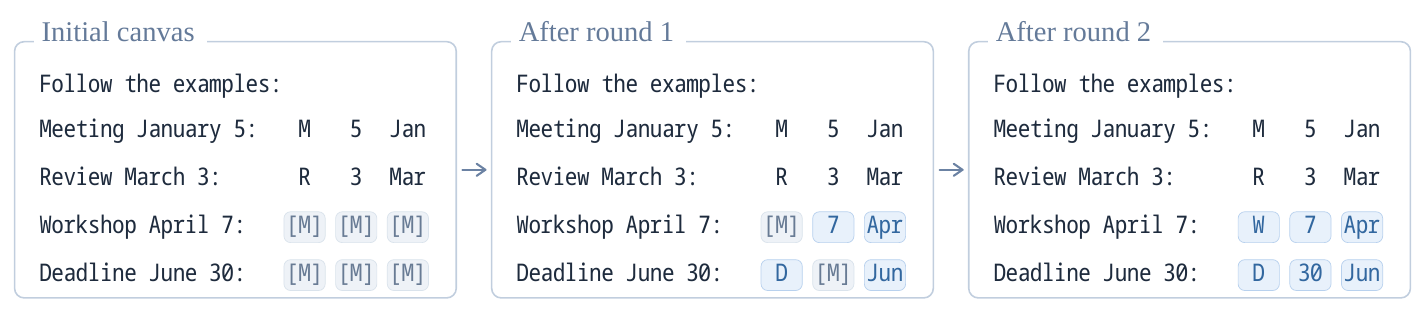}
    \caption{\textbf{Multi-target few-shot prompting with a DLM.}
    Demonstrations map event names and dates to initials, days, and abbreviated months.
    The DLM progressively fills two masked outputs within a shared canvas; tokens revealed at one site can inform subsequent predictions at the other.}
    \label{fig:introduction_teaser}
\end{figure*}

\vspace*{\stretch{2.4}}
\clearpage

\ifdefined\arxivbuild\else

\fi

\section{Introduction}
\label{sec:introduction}

Few-shot prompting lets users specify model behavior through demonstrations, without updating model parameters~\citep{brown2020language,garg2022what}.
This is useful for test-time guidance, especially when instructions or rewards are difficult to specify~\citep{zhou2023large,zamfirescupereira2023why,kwon2023reward}.
For instance, in editing by example, a demonstrated formatting change or API migration can illustrate a transformation to apply elsewhere~\citep{gupta2023grace,wei2024coeditor}.
When a demonstrated edit must be applied at multiple sites, a natural approach is to generate the tokens in place within a shared canvas containing all demonstrations and relevant context.
We call this \textbf{multi-target few-shot prompting}.

However, causal autoregressive models generate outputs as continuations, making it awkward to infill a target site using right-hand-side context.
Incorporating such context requires rearranging the prompt or using specialized infilling formats~\citep{bavarian2022efficient,fried2022incoder}, a process that becomes more cumbersome when editing multiple sites.
Moreover, pretrained associations and sensitivity to prompt design can interfere with the model's adherence to demonstrations~\citep{min2022rethinking,wei2023larger,zhao2021calibrate}.
Although guidance methods have been developed to mitigate these issues, they are primarily designed for autoregressive models and do not directly apply to in-place editing~\citep{shi2024trusting,peng2025enhancing}.

Diffusion language models (DLMs) offer a natural interface for these tasks because they are not restricted to next-token prediction and can generate tokens anywhere on the shared canvas~\citep{austin2021structured,sahoo2024simple,shi2024simplified}.
As illustrated in \Cref{fig:introduction_teaser}, the missing output regions are initially masked and progressively filled, while demonstrations and surrounding context remain visible.
Each generated token can therefore condition on both left and right context, avoiding the specialized prompting harnesses often needed in causal autoregressive generation~\citep{ou2025your}.
Moreover, jointly infilling multiple sites allows tokens revealed at one site to inform subsequent predictions at others, helping maintain consistency~\citep{kim2026dapd}.

Yet several challenges remain in using DLMs for multi-target few-shot prompting.
First, infilling requires allocating masked positions before generation, even though the required output lengths are generally unknown~\citep{xu2026predict,zhang2025flexible}.
Second, although DLMs condition on demonstrations, they may still fail to follow them reliably.
Context-aware decoding suggests that contrasting predictions across contexts could strengthen adherence to demonstrations~\citep{malkin2022coherence,shi2024trusting}, but extending such guidance to complete DLM infills requires joint likelihoods that are often intractable or decoder-specific~\citep{turok2026duel}.

In this work, we study how to overcome such challenges.
We state our contributions next.

\subsection{Contributions}
\label{sec:introduction:contributions}

\paragraph{Multi-target few-shot prompting.}
We formulate this task as infilling multiple masked regions of a shared canvas containing demonstrations and surrounding context.
We show that frozen DLMs can generate the missing outputs in place using the visible canvas directly, avoiding the specialized prompting harnesses often needed for causal autoregressive infilling.

\paragraph{Example-guided sampling.}
We construct a tractable distribution over complete multi-site infills by revealing one token at a time, allowing us to compare candidate likelihoods with and without demonstrations.
The resulting likelihood-ratio reward enables inference-time importance resampling to favor example-supported completions, offering an additional way to encourage adherence to demonstrations without retraining the generator.

\paragraph{Empirical evaluation.}
We evaluate the effects of demonstrations, target counts, example guidance, and learned length prediction across repurposed few-shot prompting benchmarks and software refactoring case studies.
We find that demonstrations generally improve infilling accuracy and guidance provides additional gains in some settings, but accurate length prediction remains a challenge.

\subsection{Related Work}
\label{sec:related_work}

\paragraph{Few-shot prompting.}
Few-shot prompting specifies tasks through demonstrations~\citep{brown2020language}, but mappings and placement can affect model behavior~\citep{min2022rethinking,wei2023larger,li2026where}.
Applications include information extraction and semantic parsing~\citep{pang2023guideline,shin2022fewshot}, as well as programming by example and code editing, where demonstrations or prior edits guide subsequent transformations~\citep{gulwani2011automating,meng2011systematic,gupta2023grace,wei2024coeditor}.
Prior work also explores batching multiple inputs into a single prompt, allowing their answers to be generated in one model call~\citep{cheng2023batch,lin2024batchprompt}.
We instead study joint infilling of multiple unresolved regions in place within a shared canvas containing demonstrations and surrounding context.

\paragraph{Test-time guidance.}
Beyond few-shot prompting, various test-time guidance methods have been developed to steer model generation.
Contrastive decoding compares predictions from different models~\citep{li2023contrastive,obrien2023contrastive}, while context-aware methods compare predictions under different conditioning contexts to strengthen the influence of supplied evidence or demonstrations~\citep{malkin2022coherence,shi2024trusting,peng2025enhancing,song2025instantly}.
In diffusion models, classifier-free guidance adjusts the influence of conditioning information~\citep{ho2022classifier}, while KL-regularized reward tilting~\citep{korbak2022rl,rafailov2023direct} and inference-time resampling~\citep{singhal2025general} provide mechanisms for reweighting candidate generations.
However, extending likelihood-based comparisons to complete DLM infills requires tractable sequence-level likelihoods under a specified decoding procedure.

\paragraph{Diffusion language models.}
Masked diffusion language models iteratively unmask positions within a shared canvas, supporting flexible generation orders and parallel decoding~\citep{austin2021structured,sahoo2024simple,shi2024simplified}.
Recent models include LLaDA~\citep{nie2025large}, Dream~\citep{ye2025dream}, Nemotron-Labs-Diffusion~\citep{fu2026nemotron}, and dQwen3.5~\citep{xue2026dqwen}, alongside diffusion-augmented autoregressive approaches such as Uno~\citep{sahoo2026unlocking}.
Standard masked infilling also requires allocating output lengths in advance, motivating work on flexible-length generation and length prediction~\citep{xu2026predict,zhang2025flexible,ding2026beyond}.
Prior work shows that decoding strategies affect generation quality and efficiency~\citep{arriola2025block,wu2025fastdllm,kim2026dapd}, while the tractability of complete-sequence likelihoods depends on the decoding procedure~\citep{turok2026duel}.

\begin{figure*}[t]
    \centering
    \includegraphics[width=\textwidth]{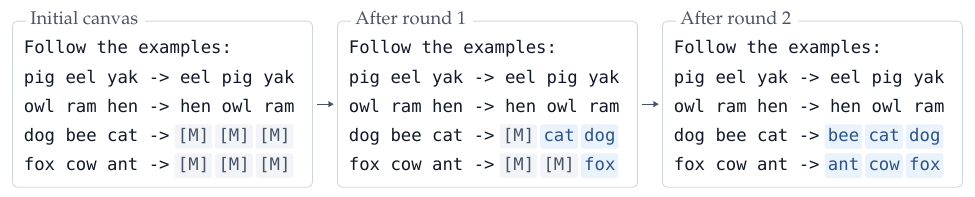}
    \caption{\textbf{Further examples of multi-target few-shot prompting.}
    Two input--output demonstrations illustrate alphabetical sorting of word lists, which a DLM applies to two missing outputs.}
    \label{fig:background_multitarget_denoising}
\end{figure*}

\ifdefined\arxivbuild\else
\begin{figure*}[t]
    \centering
    \includegraphics[width=0.95\textwidth]{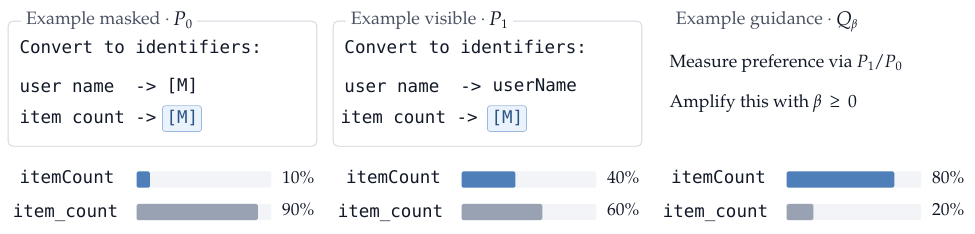}
    \caption{\textbf{Example guidance strengthens a demonstrated naming convention.}
    With the demonstration output masked (\(P_0\)), the model favors \texttt{item\_count}; revealing it (\(P_1\)) increases support for \texttt{itemCount}, but does not change the ranking.
    At an illustrative \(\beta=1\), example guidance (\(Q_\beta\)) amplifies this shift and favors \texttt{itemCount}.
    Probabilities are schematic and normalized over the two completions of the blue mask; \texttt{[M]} denotes a masked span.}
    \label{fig:methods_example_tilting}
\end{figure*}

\fi

\section{Background}
\label{sec:background}

\subsection{Few-Shot Prompting on a Shared Canvas}
\label{sec:background:few_shot_prompting}

We consider a partially observed canvas containing \(D\) demonstrations, \(J\) target sites, and relevant surrounding context.
Let \(z \in (V \cup \{m\})^N\) denote this canvas, where \(V\) is the vocabulary, \(m\) is a mask token, and \(N\) is the canvas length.
The original target positions are \(T = M(z) = \{i \in [N] : z_i = m\}\), where \([N] = \{1,\ldots,N\}\); these positions form \(J\) masked output regions.
The task is to fill these regions while preserving the visible content of the canvas.

In \Cref{fig:background_multitarget_denoising}, \(D=2\) demonstrations illustrate alphabetical sorting, while \(J=2\) missing outputs occupy six masked positions in total (\(|T|=6\) in the schematic).
At the target sites, the original word lists remain visible, while their sorted versions must be generated.
More generally, demonstrations and targets may be interleaved with surrounding context; conventional few-shot prompting is recovered when a single target follows the demonstrations.
For now, we assume the target sites are specified and their output lengths are given by an oracle, fixing the mask allocation before generation.

\subsection{Masked Diffusion Language Models}
\label{sec:background:masked_diffusion_language_models}

Like BERT-style masked language modeling~\citep{devlin2019bert}, masked DLMs predict missing tokens using visible context on both sides of each masked position~\citep{austin2021structured,sahoo2024simple,shi2024simplified}.
At each masked position \(i \in M(z)\), the model predicts a distribution \(P_{\mathrm{DLM}}(v_i \mid z)\) over candidate tokens \(v_i \in V\).
Turning these predictions into complete infills requires a decoding algorithm that determines which masked positions to reveal and when.
Common strategies select positions by prediction confidence or a prescribed order, potentially revealing multiple tokens in parallel~\citep{wu2025fastdllm,kim2026dapd,fruytier2026pig}.
As illustrated in \Cref{fig:background_multitarget_denoising}, newly revealed tokens become visible context for subsequent predictions, including those at other target sites.
Different decoding procedures can induce different distributions over complete infills, so we seek a decoder that permits exact sampling and scoring.

\ifdefined\arxivbuild\else
\begin{figure*}[t]
    \centering
    \includegraphics[width=0.95\textwidth]{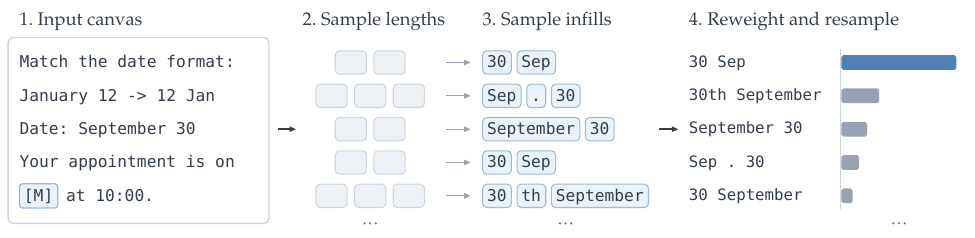}
    \caption{\textbf{Inference with unknown target lengths.}
    A length predictor runs once on the compact canvas; each sample then draws a length allocation and generates a DLM infill over the expanded masks.
    Samples receive reward-guidance weights \(w_k = (P_1(L^{(k)},u^{(k)})/P_0(L^{(k)},u^{(k)}))^\beta\), comparing their likelihoods with demonstration outputs visible (\(P_1\)) and masked (\(P_0\)).
    Resampling according to these weights approximates \(Q_\beta\).
    Bars aggregate normalized weights over identical completions; tokens, samples, and weights are schematic.}
    \label{fig:methods_pipeline}
\end{figure*}

\fi

\section{Multi-Target Few-Shot Prompting and Reward Guidance}
\label{sec:methods}

While \Cref{sec:background} assumes known output lengths, practical multi-target few-shot prompting must generally determine both the lengths and contents of missing outputs.
We formulate this joint generation problem using a shared canvas with one placeholder per target site.
To characterize the complete infills produced by a DLM, we construct tractable likelihoods under one-token-at-a-time decoding.
These likelihoods additionally enable reward guidance to encourage adherence to the demonstrations.

\subsection{Length Allocation and Joint Infilling}
\label{sec:methods_multitarget}

We begin with a compact canvas in which each missing output is represented by a single placeholder.
Let \(\widetilde z\) denote this canvas, containing \(J\) target placeholders together with the visible demonstrations and surrounding context.
For a proposed length vector \(L=(L_1,\ldots,L_J)\), where each \(L_j\geq1\) specifies the token length of target \(j\), we expand each placeholder into \(L_j\) mask tokens.
This produces an expanded canvas \(z_L\) with target positions \(T_L=M(z_L)\), on which the DLM can generate the missing contents.

We model joint generation of output lengths and contents by first allocating lengths and then infilling the resulting canvas.
Let \(\pi_\psi(L\mid\widetilde z)\) be a distribution over length allocations, with each \(L_j\) ranging over a prescribed finite set of positive integers.
For a given allocation, let \(u=(u_i)_{i\in T_L}\in V^{T_L}\) denote a complete infill, assigning a vocabulary token \(u_i\in V\) to each target position.
We define the joint distribution
\begin{equation}
    P(L,u\mid\widetilde z)
    :=
    \pi_\psi(L\mid\widetilde z)\,
    P(u\mid z_L).
    \label{eq:methods_joint_generation}
\end{equation}
The allocator may be a point mass at oracle lengths or a learned predictor, as investigated in \Cref{sec:experiments}.
To complete this formulation, we next characterize the distribution \(P(u\mid z_L)\) induced by DLM decoding.

\subsection{Tractable DLM Infilling}
\label{sec:methods_infilling}

A DLM provides predictions at individual masked positions, but different decoding procedures can induce different distributions over complete infills (\Cref{sec:background}).
To obtain tractable likelihoods, we focus on decoders that reveal one token at a time, including prescribed orders and deterministic confidence-based selection.

Fix an expanded canvas \(z=z_L\) with target positions \(T=M(z)\), and let \(u=(u_i)_{i\in T}\in V^T\) be a candidate completion.
At step \(s\), the decoder selects an unrevealed target position \(i_s\) based on the current canvas, then samples its token from the DLM.
Let \(z^{(s)}\) denote the canvas after revealing the first \(s\) tokens of \(u\), with \(z^{(0)}=z\).
For deterministic position selection, replaying the policy on \(u\) determines the reveal order and yields the complete-infill likelihood
\begin{equation}
    P(u\mid z)
    :=
    \prod_{s=1}^{|T|}
    P_{\mathrm{DLM}}(u_{i_s}\mid z^{(s-1)}).
    \label{eq:background_any_order}
\end{equation}
The following lemma establishes the properties of this construction.

\begin{restatable}[One-token-at-a-time decoding]{lemma}{fixedOrderDecoding}
\label{lem:background_any_order}
For any deterministic one-token-at-a-time position-selection policy, \Cref{eq:background_any_order} defines a normalized distribution \(P(u\mid z)\) that can be sampled and scored exactly.
If the DLM predicts exact data conditionals, then \(P(u\mid z)\) recovers the conditional data distribution wherever defined, independently of the reveal policy~\citep{ou2025your}.
\end{restatable}
\ifdefined\arxivbuild
\begin{figure*}[t]
    \centering
    \includegraphics[width=0.95\textwidth]{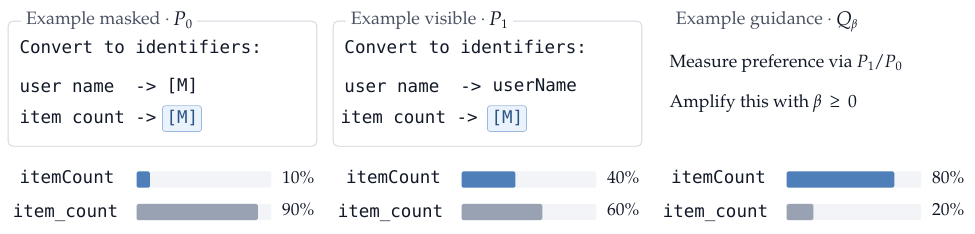}
    \caption{\textbf{Example guidance strengthens a demonstrated naming convention.}
    With the demonstration output masked (\(P_0\)), the model favors \texttt{item\_count}; revealing it (\(P_1\)) increases support for \texttt{itemCount}, but does not change the ranking.
    At an illustrative \(\beta=1\), example guidance (\(Q_\beta\)) amplifies this shift and favors \texttt{itemCount}.
    Probabilities are schematic and normalized over the two completions of the blue mask; \texttt{[M]} denotes a masked span.}
    \label{fig:methods_example_tilting}
\end{figure*}

\fi

Exact sampling and scoring hold for the induced distribution regardless of whether the DLM matches the data conditionals.
Recovery of the data distribution requires exact conditional predictions; learned DLMs need not satisfy this assumption, and their induced distributions may depend on the decoding policy.
While our construction excludes parallel token revelation, it permits generating and scoring only the designated targets \(T\) while other masked positions remain unresolved.

\subsection{Reward Guidance}
\label{sec:methods_sampling}

Even with demonstrations visible, a DLM may favor completions that do not follow their demonstrated patterns, as illustrated in \Cref{fig:methods_example_tilting}.
The tractable likelihoods above allow us to measure how strongly the demonstrations support each candidate and use this signal for additional reward guidance.
Let \(E\subseteq[N]\setminus T\) denote selected demonstration-output positions, and write \(z\setminus E\) for the canvas with these positions additionally masked.
We compare the likelihood of the same candidate with and without these demonstration outputs visible:
\[
    P_1(u):=P(u\mid z),
    \qquad
    P_0(u):=P(u\mid z\setminus E).
\]
Both distributions use the same decoding policy over the original target positions \(T\), leaving the additional masks in \(E\) unresolved.
Their log-likelihood ratio defines the reward
\begin{equation}
    R(u)
    :=
    \log\frac{P_1(u)}{P_0(u)}.
    \label{eq:methods_reward}
\end{equation}
A higher reward indicates greater relative support from the demonstrations, rather than necessarily greater correctness.
We incorporate this reward into generation by defining the guided distribution
\begin{equation}
    Q_\beta(u)
    :=
    \frac{P_1(u)\,\exp(\beta R(u))}{Z_\beta},
    \label{eq:methods_tilting}
\end{equation}
where \(\beta\geq0\) controls guidance strength and \(Z_\beta\) is the normalizing constant~\footnote{This form of reward guidance is also known as exponential tilting~\citep{korbak2022rl}}.
At \(\beta=0\), we recover unguided sampling from \(P_1\), which already conditions on the demonstrations.
Increasing \(\beta\) favors candidates with higher rewards, potentially counteracting a pretrained model's preference for familiar but demonstration-inconsistent completions.

\ifdefined\arxivbuild
\begin{figure*}[t]
    \centering
    \includegraphics[width=0.95\textwidth]{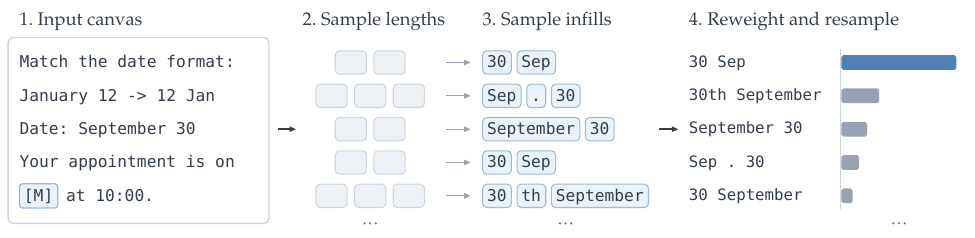}
    \caption{\textbf{Inference with unknown target lengths.}
    A length predictor runs once on the compact canvas; each sample then draws a length allocation and generates a DLM infill over the expanded masks.
    Samples receive reward-guidance weights \(w_k = (P_1(L^{(k)},u^{(k)})/P_0(L^{(k)},u^{(k)}))^\beta\), comparing their likelihoods with demonstration outputs visible (\(P_1\)) and masked (\(P_0\)).
    Resampling according to these weights approximates \(Q_\beta\).
    Bars aggregate normalized weights over identical completions; tokens, samples, and weights are schematic.}
    \label{fig:methods_pipeline}
\end{figure*}

\fi

However, direct sampling from \(Q_\beta\) is generally intractable because its sequence-level reward does not yield tractable token-level conditionals.
Instead, we approximate reward-guided generation using importance resampling from \(P_1\).
We draw \(K\) independent candidates \(u^{(k)}\sim P_1\), evaluate their likelihoods under \(P_0\), and resample an index \(I\) according to
\begin{equation}
    \Pr(I=k\mid\text{pool})
    =
    \frac{w_k}{\sum_{h=1}^{K}w_h},
    \qquad
    w_k=\exp(\beta R(u^{(k)})).
    \label{eq:methods_resampling}
\end{equation}
These weights are proportional to \(Q_\beta/P_1\), allowing finite-pool resampling to approximate the guided distribution without modifying the DLM.
The following proposition formalizes the KL-regularized interpretation of reward guidance and the consistency of its approximation.

\begin{restatable}[Reward guidance and importance sampling]{proposition}{exampleTilting}
\label{prop:methods_sampling}
Let \(P_0\) and \(P_1\) be positive distributions on a common finite candidate space, and let \(\beta\geq0\).
The distribution \(Q_\beta\) defined by \Cref{eq:methods_tilting} uniquely solves
\begin{equation}
    Q_\beta
    =
    \argmax{Q}
    \braces[\big]{
        \beta\,\expval{U\sim Q} R(U)
        - \KL{Q}{P_1}
    }.
    \restatelabel{eq:methods_kl_characterization}
\end{equation}
Moreover, the weighted empirical distribution obtained from independent samples of \(P_1\) using \Cref{eq:methods_resampling} converges almost surely to \(Q_\beta\) in total variation as \(K\to\infty\).
\end{restatable}

\ifdefined\arxivbuild\else
\begin{figure*}[t]
\centering
\includegraphics[width=\textwidth]{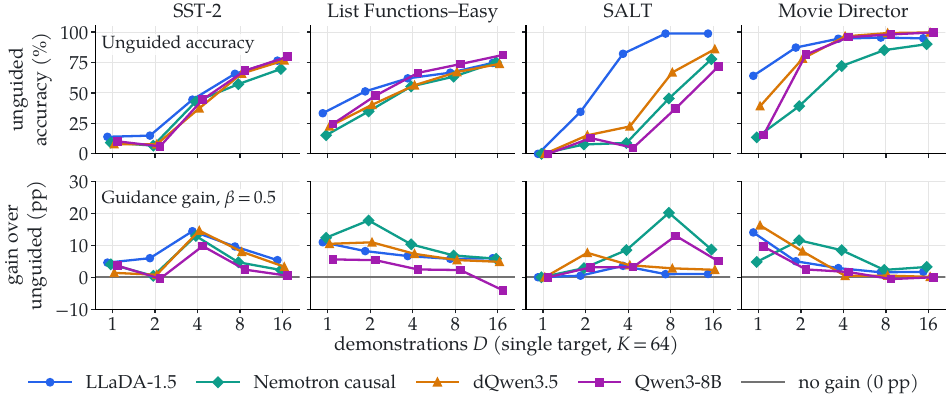}
\caption{\textbf{Single-target prompting as the number of demonstrations grows.}
The top row shows unguided answer accuracy (\(\beta=0\)), while the bottom row shows the accuracy gain from example guidance (\(\beta=0.5\)), obtained by reweighting the same \(K=64\) samples per canvas.
Four benchmarks and four model interfaces, with 64 prompt families per benchmark; DLMs use oracle target lengths, while Qwen3 generates freely.}
\label{fig:completion_evidence}
\end{figure*}

\fi

\paragraph{Extension to unknown lengths.}
Reward guidance also applies to joint assignments of output lengths and contents.
Let \(E_L\) denote the selected demonstration-output positions in the expanded canvas \(z_L\).
Using the factorization from \Cref{eq:methods_joint_generation}, we define
\begin{equation}
    P_1(L,u)
    :=
    \pi_\psi(L\mid\widetilde z)\,
    P(u\mid z_L),
    \qquad
    P_0(L,u)
    :=
    \pi_\psi(L\mid\widetilde z)\,
    P(u\mid z_L\setminus E_L).
    \label{eq:methods_joint_lengths}
\end{equation}
Since both contexts use the same allocator, its contribution cancels from the reward \(R(L,u)=\log[P_1(L,u)/P_0(L,u)]\), although resampling joint candidates can still change the marginal distribution over lengths.
The same guidance procedure and \Cref{prop:methods_sampling} apply to these joint candidates (see \Cref{app:proofs:joint_length}).
\Cref{fig:methods_pipeline} summarizes the complete procedure, from length allocation and DLM infilling to optional reward-guided resampling.

\ifdefined\arxivbuild
\begin{figure*}[t]
\centering
\includegraphics[width=\textwidth]{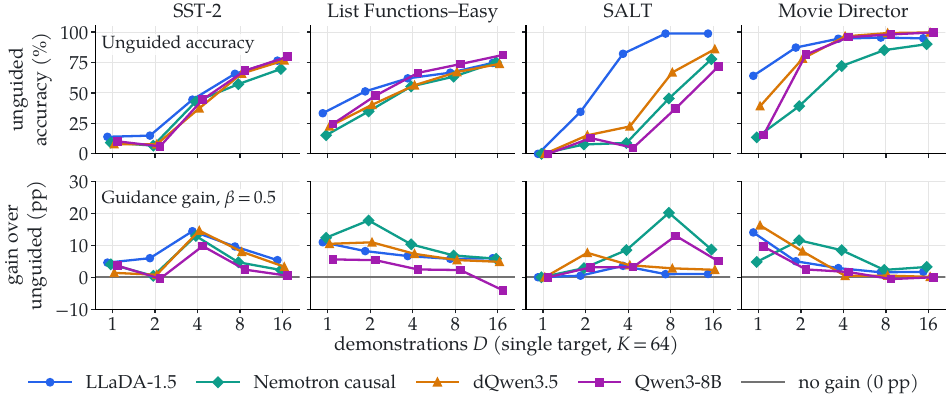}
\caption{\textbf{Single-target prompting as the number of demonstrations grows.}
The top row shows unguided answer accuracy (\(\beta=0\)), while the bottom row shows the accuracy gain from example guidance (\(\beta=0.5\)), obtained by reweighting the same \(K=64\) samples per canvas.
Four benchmarks and four model interfaces, with 64 prompt families per benchmark; DLMs use oracle target lengths, while Qwen3 generates freely.}
\label{fig:completion_evidence}
\end{figure*}

\fi

\section{Few-Shot Prompting Experiments with Standard Benchmarks}
\label{sec:experiments}

We first quantify the effects of demonstrations and example guidance in single-target infilling (\Cref{sec:experiments:completion}).
We then investigate how these effects change when multiple targets are infilled jointly (\Cref{sec:experiments:multitarget}).
Finally, we evaluate the impact of predicting target lengths rather than supplying them through an oracle (\Cref{sec:experiments:lengths}).
We defer additional experiments to~\Cref{app:additional_experiments}, such as on guidance strength (\Cref{app:guidance_strength}) and likelihood sharpening (\Cref{app:power_details}).

\ifdefined\arxivbuild\else
\begin{figure*}[t]
\centering
\includegraphics[width=\textwidth]{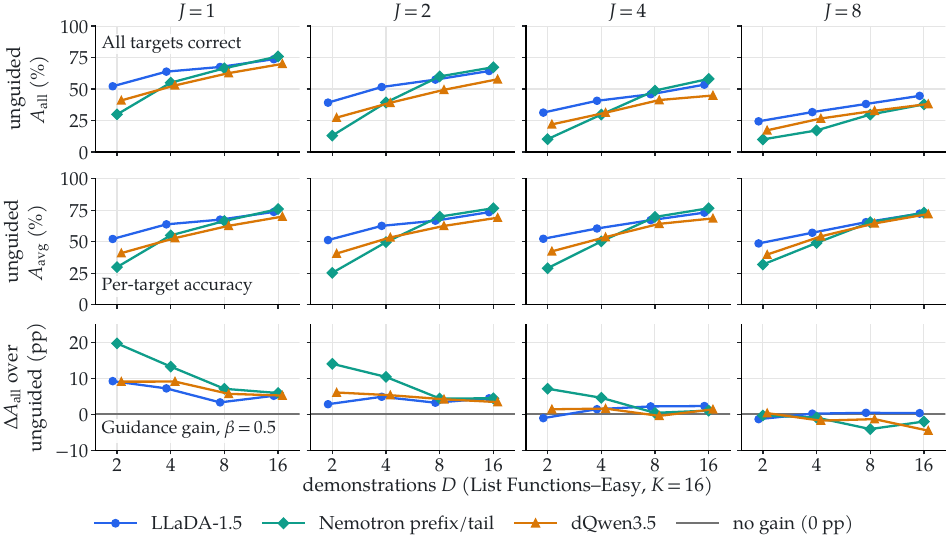}
\caption{\textbf{Demonstrations and target count on List Functions--Easy.}
Each column fixes the number of jointly infilled targets \(J\), while the horizontal axis varies the number of demonstrations \(D\).
The rows show all-target accuracy \(A_{\mathrm{all}}\) and average per-target accuracy \(A_{\mathrm{avg}}\) under unguided sampling (\(\beta=0\)), followed by the gain in \(A_{\mathrm{all}}\) from example guidance (\(\beta=0.5\)) using the same samples.
Three DLMs, oracle target lengths, and \(K=16\) samples per canvas.}
\label{fig:multitarget_evidence}
\end{figure*}

\fi

\paragraph{Experimental setup.}
We evaluate four benchmarks spanning different few-shot transformations: SST-2~\citep{socher2013recursive,wang2018glue} with demonstration-defined sentiment codes, List Functions--Easy~\citep{srivastava2022beyond,rule2020child} for list transformations, SALT repetition~\citep{zheng2025curse,zheng2025logidynamics} for artificial-language transformations, and Movie Director~\citep{liu2012conversational,zhao2021calibrate} for information extraction.
We use three DLMs: LLaDA-1.5~\citep{zhu2026llada}, Nemotron-Labs-Diffusion-8B~\citep{fu2026nemotron}, and dQwen3.5-9B-Base~\citep{xue2026dqwen}, together with Qwen3-8B~\citep{yang2025qwen3} as an autoregressive single-target comparison.

Throughout, \(D\) and \(J\) denote the numbers of demonstrations and targets, respectively.
We refer to generation without additional guidance (\(\beta=0\)) as \emph{standard infilling} for DLMs, or equivalently \emph{unguided sampling}, and compare this with example guidance at \(\beta=0.5\) by default.
Outputs are graded by task-specific exact or parsed matching, and the DLM experiments reveal one target token at a time in ascending canvas order.
Further experimental details appear in \Cref{app:experiment_setup}.

\subsection{Single-Target Few-Shot Prompting}
\label{sec:experiments:completion}

We begin with single-target infilling, which closely resembles conventional autoregressive few-shot prompting, except that DLMs generate a masked output rather than a continuation.
Since few-shot prompting is sensitive to demonstration selection and ordering~\citep{zhao2021calibrate}, we hold the target fixed when comparing different demonstration counts.
We first randomly draw a sequence of \(16\) demonstrations from the benchmark, followed by one target.
For each \(D\in\{1,2,4,8,16\}\), we construct a prompt by retaining only the first \(D\) demonstrations while keeping the target unchanged.
These five related prompts form a \emph{prompt family}, allowing us to vary the demonstration count without changing the target or previously included demonstrations.
We evaluate \(64\) families per benchmark across four tasks and four model interfaces, drawing \(K=64\) samples for each prompt.
\Cref{fig:completion_evidence} reports standard infilling accuracy and the gain from guidance.

Increasing the demonstration count from \(D=1\) to \(D=16\) improves accuracy without additional guidance in all \(16\) model--task series (\Cref{fig:completion_evidence}, top).
For LLaDA, unguided accuracy rises from \(14.0\%\) to \(76.5\%\) on SST-2, from \(33.3\%\) to \(74.7\%\) on List Functions--Easy, from \(0.0\%\) to \(98.8\%\) on SALT repetition, and from \(64.1\%\) to \(94.9\%\) on Movie Director.
This shows that demonstrations provide useful conditioning information even before applying reward guidance.
Consistent with this observation, the reference answer's example reward generally increases with \(D\), indicating stronger relative likelihood support from the demonstrations (\Cref{app:single_target_details}).

Beyond these improvements from demonstration conditioning, example guidance yields further accuracy gains in several settings, particularly SST-2 at intermediate demonstration counts (\Cref{fig:completion_evidence}, bottom).
At \(D=4\), guidance improves LLaDA's accuracy by \(14.4\) percentage points on SST-2, \(6.6\) on List Functions--Easy, \(3.6\) on SALT repetition, and \(2.9\) on Movie Director, where unguided accuracy already exceeds \(94\%\).
These gains tend to shrink as standard infilling accuracy saturates, and depend on whether the reward favors correct samples.

\ifdefined\arxivbuild\else
\begin{figure*}[t]
\centering
\includegraphics[width=\textwidth]{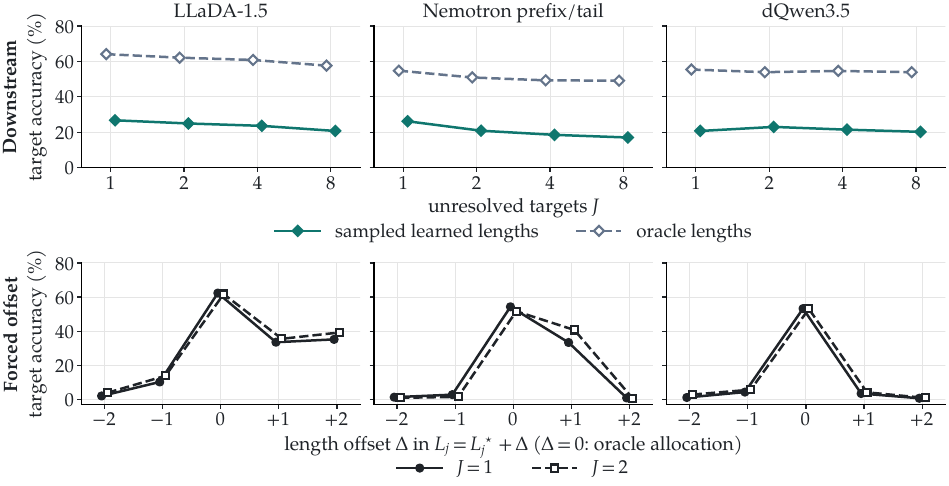}
\caption{\textbf{Length allocation and downstream infilling accuracy.}
On List Functions--Easy (\(D=4\), \(\beta=0\)), the top row compares per-target accuracy \(A_{\mathrm{avg}}\) under oracle and sampled learned lengths as the number of targets \(J\) increases.
The bottom row measures sensitivity to length errors, setting each target length to \(L_j=L_j^\star+\Delta\) for \(J\in\{1,2\}\), where \(L_j^\star\) is the reference length and \(\Delta=0\) recovers oracle allocation.
These forced offsets are distinct from the predictor's own error distribution.
Results cover three DLMs; additional length-prediction diagnostics appear in \Cref{app:length_details}.}
\label{fig:predicted_length_infilling}
\end{figure*}

\fi

\subsection{Multi-Target Few-Shot Prompting}
\label{sec:experiments:multitarget}

We next examine how demonstration conditioning and example guidance behave when multiple targets are infilled jointly.
For each prompt family, we randomly draw \(16\) demonstrations and \(8\) targets from List Functions--Easy, then interleave them into a shared sequence.
For each configuration \(D\in\{2,4,8,16\}\) and \(J\in\{1,2,4,8\}\), we construct a canvas by removing all but the first \(D\) demonstrations and first \(J\) targets from this sequence.
The retained items preserve their relative order, allowing us to vary demonstration and target counts without independently reshuffling the canvas.
For each resulting canvas, the DLM draws \(K=16\) joint samples using oracle target lengths, which we evaluate under standard infilling and additional example guidance.

\ifdefined\arxivbuild
\begin{figure*}[t]
\centering
\includegraphics[width=\textwidth]{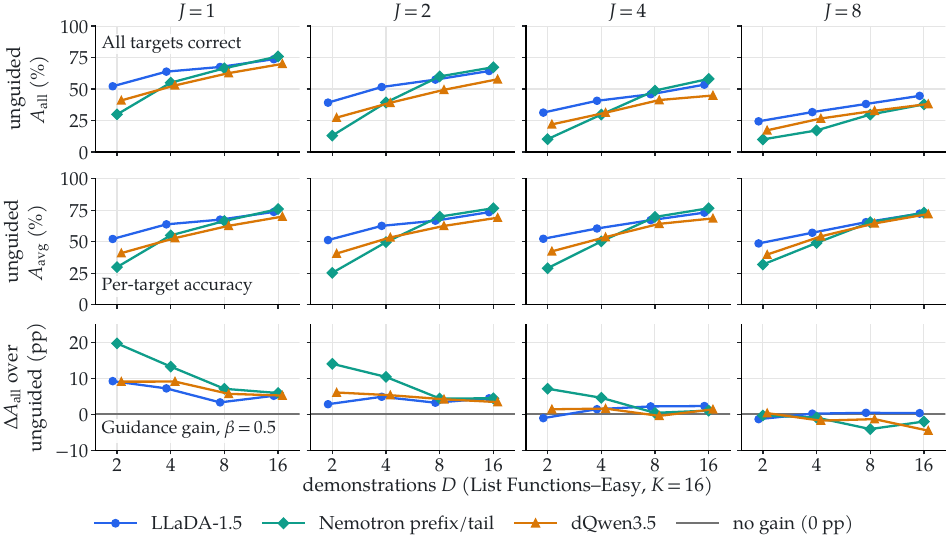}
\caption{\textbf{Demonstrations and target count on List Functions--Easy.}
Each column fixes the number of jointly infilled targets \(J\), while the horizontal axis varies the number of demonstrations \(D\).
The rows show all-target accuracy \(A_{\mathrm{all}}\) and average per-target accuracy \(A_{\mathrm{avg}}\) under unguided sampling (\(\beta=0\)), followed by the gain in \(A_{\mathrm{all}}\) from example guidance (\(\beta=0.5\)) using the same samples.
Three DLMs, oracle target lengths, and \(K=16\) samples per canvas.}
\label{fig:multitarget_evidence}
\end{figure*}

\fi

With multiple targets, we distinguish between completing every target correctly and maintaining accuracy at individual sites.
Let \(c_{kj}\in\{0,1\}\) indicate whether target \(j\) is correct in sample \(k\), and let \(\widetilde w_k\) denote its normalized resampling weight.
We define all-target accuracy \(A_{\mathrm{all}}\) and average per-target accuracy \(A_{\mathrm{avg}}\) by
\begin{equation*}
    A_{\mathrm{all}}
    := \sum_{k=1}^{K}\widetilde w_k\prod_{j=1}^{J}c_{kj},
    \qquad
    A_{\mathrm{avg}}
    := \sum_{k=1}^{K}\widetilde w_k\,
        \parens*{\frac{1}{J}\sum_{j=1}^{J}c_{kj}}.
\end{equation*}
For each sample pool, \(A_{\mathrm{all}}\) is the probability that resampling selects an infill with every target correct, whereas \(A_{\mathrm{avg}}\) is the expected fraction of correct targets.
We average these quantities across prompt families; they coincide when \(J=1\).
\Cref{fig:multitarget_evidence} reports both under standard infilling and the guidance gain in \(A_{\mathrm{all}}\) at \(\beta=0.5\) (\Cref{app:experiment_setup}).

\ifdefined\arxivbuild
\begin{figure*}[t]
\centering
\includegraphics[width=\textwidth]{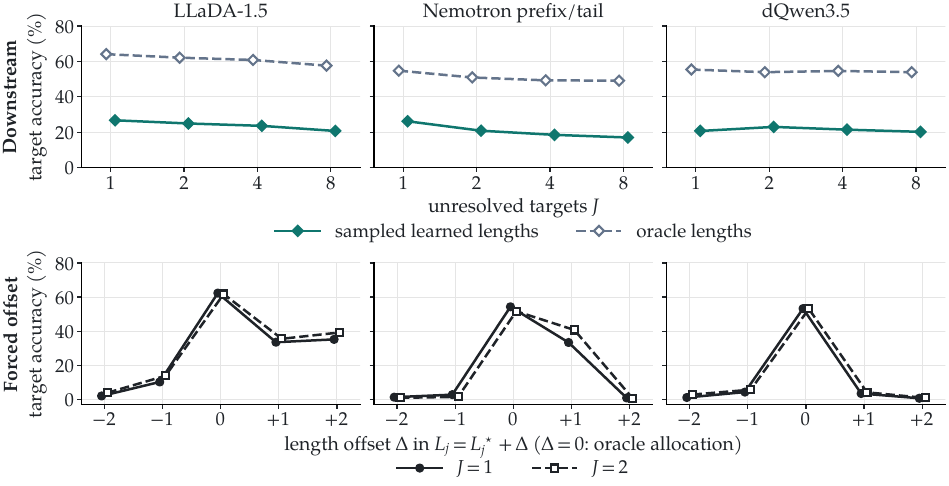}
\caption{\textbf{Length allocation and downstream infilling accuracy.}
On List Functions--Easy (\(D=4\), \(\beta=0\)), the top row compares per-target accuracy \(A_{\mathrm{avg}}\) under oracle and sampled learned lengths as the number of targets \(J\) increases.
The bottom row measures sensitivity to length errors, setting each target length to \(L_j=L_j^\star+\Delta\) for \(J\in\{1,2\}\), where \(L_j^\star\) is the reference length and \(\Delta=0\) recovers oracle allocation.
These forced offsets are distinct from the predictor's own error distribution.
Results cover three DLMs; additional length-prediction diagnostics appear in \Cref{app:length_details}.}
\label{fig:predicted_length_infilling}
\end{figure*}

\fi

Demonstrations continue to improve joint infilling, with standard all-target accuracy increasing at every successive demonstration count (\Cref{fig:multitarget_evidence}, top).
For LLaDA with \(J=8\), \(A_{\mathrm{all}}\) rises from \(24.3\%\) at \(D=2\) to \(44.5\%\) at \(D=16\).
Increasing the number of targets, however, affects all-target and per-target accuracy differently.
At \(D=16\), LLaDA maintains nearly the same \(A_{\mathrm{avg}}\) as \(J\) grows from \(1\) to \(8\) (\(73.6\%\) versus \(72.3\%\)), despite a decline in all-target accuracy (\Cref{fig:multitarget_evidence}, top and middle).
Thus, lower joint success need not imply a comparable decline in accuracy at individual target sites.

The benefits of additional example guidance, however, are less consistent as the number of targets grows.
On List Functions--Easy, guidance at \(\beta=0.5\) yields positive gains with confidence intervals above zero in every \(J=1\) condition, but in none of the \(J=8\) conditions (\Cref{fig:multitarget_evidence}, bottom).
This pattern is not universal across tasks: a separate cross-benchmark study at \(D=4\), \(K=64\), and \(J=8\) finds guidance gains on SST-2 but losses on Movie Director (\Cref{app:mj_grid}).
Together, these results indicate that guidance effectiveness depends on both the task and the number of jointly infilled targets.

\ifdefined\arxivbuild\else
\begin{figure*}[t]
    \centering
    \includegraphics[width=\textwidth]{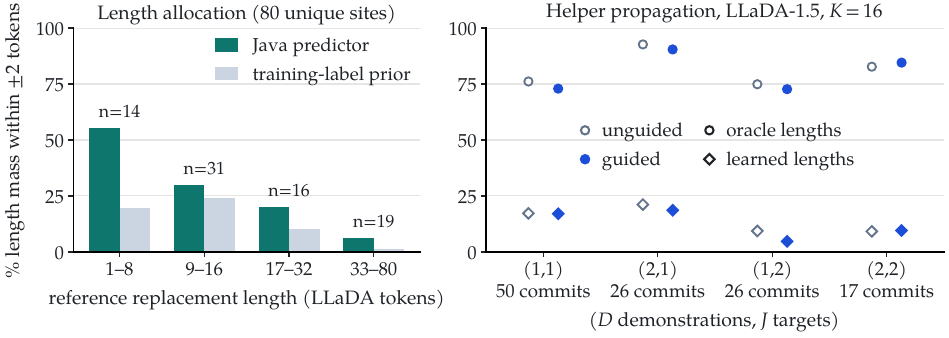}
    \caption{\textbf{Java helper refactoring with LLaDA-1.5} ($K=16$).
    \emph{Left:} mass of the Java length predictor within two tokens of the reference length versus a training-label prior, by reference length.
    \emph{Right:} joint LLM-judged success by $(D,J)$, oracle or learned lengths, and unguided ($\beta=0$, hollow) or example-guided ($\beta=0.5$, filled) evaluation of the same sample pools.
    Pooled contrasts are in \Cref{tab:case_study_results_by_cell}.}
    \label{fig:case_studies:results}
\end{figure*}

\fi

\subsection{Predicting Target Lengths}
\label{sec:experiments:lengths}

The preceding experiments assume oracle target lengths, whereas infilling in practice generally requires predicting them.
To evaluate the learned-allocation setting of \Cref{sec:methods_multitarget}, we train length predictors on structured few-shot tasks and Java edits (\Cref{app:length_details}).
On List Functions--Easy, we fix \(D=4\) and vary \(J\in\{1,2,4,8\}\) across three DLMs, comparing oracle lengths against sampled predictions with the same content generator and standard infilling (\(\beta=0\)).
\Cref{fig:predicted_length_infilling} reports downstream accuracy and sensitivity to length errors.

The learned predictors capture useful information about the required output lengths.
They assign \(40\)--\(46\%\) probability mass to the reference length, compared with approximately \(27\%\) for a prior constructed from visible demonstration lengths (\Cref{app:length_details}).
This indicates that the compact canvas contains useful predictive information about the missing output lengths.

Despite this improvement over the length prior, learned allocation remains a gap for reliable downstream infilling.
Sampling predicted lengths retains only \(35\)--\(48\%\) of oracle-length per-target accuracy (\Cref{fig:predicted_length_infilling}, top).
To examine this sensitivity directly, we impose offsets \(L_j=L_j^\star+\Delta\), where \(L_j^\star\) is the reference token length.
One-token underallocation harms LLaDA and Nemotron more than overallocation, while dQwen3.5 is sensitive to errors in either direction (\Cref{fig:predicted_length_infilling}, bottom).
These findings identify accurate length allocation as a remaining challenge for practical multi-target infilling.

\section{Case Studies in Software Refactoring}
\label{sec:case_studies}

Software refactoring often requires applying a demonstrated transformation at multiple sites.
For example, updating callers after extracting a helper requires adapting receivers, arguments, and surrounding code rather than copying an accepted edit.

\ifdefined\arxivbuild
\begin{figure*}[t]
    \centering
    \includegraphics[width=\textwidth]{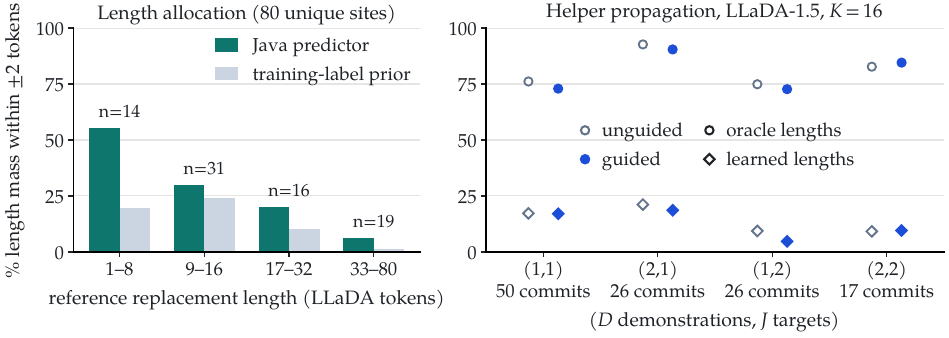}
    \caption{\textbf{Java helper refactoring with LLaDA-1.5} ($K=16$).
    \emph{Left:} mass of the Java length predictor within two tokens of the reference length versus a training-label prior, by reference length.
    \emph{Right:} joint LLM-judged success by $(D,J)$, oracle or learned lengths, and unguided ($\beta=0$, hollow) or example-guided ($\beta=0.5$, filled) evaluation of the same sample pools.
    Pooled contrasts are in \Cref{tab:case_study_results_by_cell}.}
    \label{fig:case_studies:results}
\end{figure*}

\fi

\subsection{Helper Refactoring Experiments}
\label{sec:case_studies:experiments}

We construct multi-site editing tasks from Extract Method commits in SWE-Refactor~\citep{xu2026swerefactor}.
Each task includes the extracted helper, \(D\) accepted \texttt{<old>}/\texttt{<new>} demonstrations, and \(J\) masked replacement sites.
We evaluate 127 tasks from 50 commits across 14 repositories, with \((D,J)\in\{1,2\}^2\); different configurations have different eligible commits (\Cref{app:case_study_details:construction}).
Frozen LLaDA-1.5~\citep{zhu2026llada} generates \(K=16\) rewrites per task using either oracle lengths or lengths sampled from a separately trained Java predictor.
We compare unguided sampling (\(\beta=0\)) with example guidance (\(\beta=0.5\)) on the same samples.
We assess joint correctness using Claude Sonnet 5.5, without executing the generated code.

A separate post-hoc ablation on 80 tasks compares generation with and without one visible accepted edit, holding the other prompt components and oracle lengths fixed.
Adding the demonstration improves joint LLM-judged success by 14.7 and 17.5 percentage points for \(J=1\) and \(J=2\), respectively (\Cref{app:case_study_details:zero_demo}).
These gains indicate that the DLM can use a demonstrated refactoring to guide edits at other sites.
Under unguided sampling, pooled joint success reaches 74.4\% with oracle lengths but only 15.8\% with predicted lengths (\Cref{tab:case_study_results_by_cell}).
This gap identifies length allocation as a major opportunity for improvement, even though the predictor outperforms a simple length prior (\Cref{fig:case_studies:results}, left).
Example guidance at \(\beta=0.5\) reduces success to 71.5\% and 14.5\%, respectively (\Cref{tab:case_study_results_by_cell}).
Thus, demonstrations help the DLM generate appropriate rewrites, but the additional likelihood-ratio reward does not improve their selection in this experiment.

\subsection{Qualitative Analysis}
\label{sec:case_studies:qualitative}

The limited benefits of guidance may arise because the sampled rewrites offer little room for improvement, or because the reward fails to identify correct candidates.
With oracle lengths, only 23 of 127 sample pools contain both judged-correct and judged-incorrect rewrites, so even a perfect selector could improve pooled success by at most 8.3 percentage points.
With learned lengths, 72 of 127 pools contain both outcomes, and the corresponding upper bound rises to 45.2 points (\Cref{app:case_study_details:cases}).
Thus, limited selection opportunities help explain the oracle-length results, whereas learned-length sampling leaves substantial room for better selection.

A post-hoc Checkstyle example illustrates how the reward can favor an incorrect rewrite.
An incorrect edit to \texttt{logMessage} changes the logged line number and message type, yet receives a higher example reward than the developer-consistent rewrite (8.69 versus 2.05 nats).
At \(\beta=0.5\), this gives the incorrect rewrite approximately 28 times the resampling weight of each correct occurrence.
Despite appearing only once among \(K=16\) samples, compared with seven occurrences of the correct rewrite, it receives 61.8\% of the total weight, reducing expected LLM-judged success from 43.8\% to 15.7\%.
Here, correct rewrites are available, but the reward disproportionately favors an incorrect one (\Cref{app:case_study_details:cases}).

\section{Conclusion}
\label{sec:conclusion}

We introduce and study multi-target few-shot prompting with masked DLMs, where demonstrations guide the infilling of multiple regions within a shared canvas.
One-token-at-a-time decoding with deterministic position selection defines a normalized distribution over complete infills, enabling exact sampling, scoring, and likelihood-based example guidance.
Across few-shot prompting benchmarks and Java refactoring, we find that demonstrations improve infilling accuracy, with example guidance providing further gains in several settings.
Our experiments further investigate the benefits of example guidance relative to likelihood sharpening, while identifying opportunities to improve target-length allocation for practical editing.
Together, these findings provide a foundation for more effective example-guided multi-site infilling.

\paragraph{Limitations.}
Our framework assumes that target sites and demonstration-output positions \(E\) are specified in advance, leaving their automatic identification for future work.
The deterministic, one-token-at-a-time decoding used for exact sampling and scoring does not exploit parallel generation, while the induced distribution may depend on the reveal order and model interface for learned DLMs.
Moreover, target lengths must be allocated before infilling, making improved length prediction an important direction for practical use.
Example guidance measures relative support from demonstrations rather than correctness, so its effectiveness depends on how the reward ranks sampled infills.
Increasing guidance strength can also concentrate resampling weights without necessarily improving accuracy (\Cref{app:guidance_strength}).
Our experiments include an autoregressive comparison for single-target prompting, while systematic comparison for multi-target infilling remains future work.
Finally, the Java refactoring study uses an LLM judge rather than code execution, includes commits examined during development, and contains some prompts in which reference text appears elsewhere in the context.

\bibliography{references}

\clearpage
\appendix

\section{Proofs}
\label{app:proofs}

We prove the one-token-at-a-time decoding and reward-guidance results of \Cref{sec:methods}, then extend the construction to unknown target lengths.

\subsection{One-Token-at-a-Time Decoding}
\label{app:proofs:fixed_order}

\fixedOrderDecoding*

\begin{proof}[Proof of \Cref{lem:background_any_order}]
Fix a canvas \(z\) and a designated set \(T\subseteq M(z)\) of masked target positions.
Let \(O=\{i\in[N]:z_i\in V\}\) denote the initially visible positions, so that any masked positions outside \(T\) remain unresolved.
We consider a deterministic decoding policy that selects one unrevealed target position based on the current canvas, then samples its token from the DLM.
For any complete assignment \(u\in V^T\), replaying this policy determines a unique sequence of selected positions \(i_1,\ldots,i_{|T|}\), which may depend on previously revealed tokens.

\emph{Normalization and exact sampling.}
Consider a decision tree representing all possible decoding trajectories.
Each node corresponds to a partially revealed canvas, and the decoding policy selects the next target position to reveal.
The outgoing edges correspond to the possible tokens at that position, with probabilities given by the DLM that sum to one.
Since the policy is deterministic and eventually reveals every target position, each complete assignment \(u\in V^T\) corresponds to exactly one leaf.
The probability of reaching that leaf is the product of its edge probabilities, precisely \(P(u\mid z)\) in \Cref{eq:background_any_order}.
Summing the probabilities of all leaves therefore gives
\[
    \sum_{u\in V^T} P(u\mid z)=1.
\]
Thus, \(P(\cdot\mid z)\) is a normalized distribution, and sequential decoding samples exactly from it.
Moreover, replaying the policy on any candidate \(u\) recovers its reveal trajectory and permits exact evaluation of \(P(u\mid z)\).

\emph{Recovery of the data conditional distribution.}
Now suppose the DLM predicts exact data conditionals, and let \(X\sim P_{\mathrm{data}}\) be a clean sequence.
We assume that the initially visible context has positive probability, \(P_{\mathrm{data}}(X_O=z_O)>0\).
For a fixed candidate \(u\), define the \emph{agreement events} \(A_s\) as follows:
\[
    A_0 := \{X_O=z_O\},
    \qquad
    A_s := A_{s-1}\cap\{X_{i_s}=u_{i_s}\}.
\]
Thus, \(A_s\) is the event that a clean data sample agrees with the visible context and the first \(s\) revealed tokens of \(u\).
Because the reveal policy selects \(i_s\) using only previously revealed information, the exact-conditional assumption implies
\[
    P_{\mathrm{DLM}}(u_{i_s}\mid z^{(s-1)})
    =
    P_{\mathrm{data}}(A_s\mid A_{s-1})
    =
    \frac{P_{\mathrm{data}}(A_s)}
         {P_{\mathrm{data}}(A_{s-1})},
\]
whenever \(P_{\mathrm{data}}(A_{s-1})>0\).

If \(P_{\mathrm{data}}(A_{|T|})>0\), multiplying these conditional probabilities yields a telescoping product:
\[
    P(u\mid z)
    =
    \prod_{s=1}^{|T|}
    \frac{P_{\mathrm{data}}(A_s)}
         {P_{\mathrm{data}}(A_{s-1})}
    =
    \frac{P_{\mathrm{data}}(A_{|T|})}
         {P_{\mathrm{data}}(A_0)}
    =
    P_{\mathrm{data}}(X_T=u\mid X_O=z_O).
\]
The final equality holds because \(A_{|T|}\) specifies every target token in addition to the initially visible context.

If instead \(P_{\mathrm{data}}(A_{|T|})=0\), let \(s^\star\) be the first step at which \(P_{\mathrm{data}}(A_{s^\star})=0\).
The preceding event has positive probability, so the corresponding DLM conditional assigns probability zero to \(u_{i_{s^\star}}\).
Consequently, \(P(u\mid z)=0\), matching the data conditional.
Thus, the equality holds for every complete assignment \(u\), independently of the reveal policy.
\end{proof}

The argument covers prescribed orders and deterministic adaptive policies, including confidence-based position selection performed before sampling each token.
It also allows masked positions outside the designated targets \(T\) to remain unresolved, as in the demonstration-masked likelihood \(P_0\).
Randomized position selection may require marginalizing over multiple reveal trajectories, while parallel token revelation is outside the present construction.

\subsection{Reward Guidance and Importance Resampling}
\label{app:proofs:tilting}

\exampleTilting*

\begin{proof}[Proof of \Cref{prop:methods_sampling}]
\emph{Variational characterization.}
Since \(P_0\) and \(P_1\) are positive distributions on the finite candidate space \(\mathcal U\), the reward \(R\) is finite and the normalizing constant satisfies
\[
    Z_\beta
    =
    \expval{U\sim P_1} e^{\beta R(U)}
    \in (0,\infty).
\]
For any distribution \(Q\) on \(\mathcal U\), substituting
\(\log Q_\beta(u)=\log P_1(u)+\beta R(u)-\log Z_\beta\)
into the definition of KL divergence gives
\[
    \beta\,\expval{U\sim Q}R(U)
    - \KL{Q}{P_1}
    =
    \log Z_\beta-\KL{Q}{Q_\beta}.
\]
Since KL divergence is nonnegative and vanishes only when its arguments agree, the objective in \Cref{eq:methods_kl_characterization} is uniquely maximized by \(Q=Q_\beta\).
In particular, setting \(\beta=0\) recovers \(Q_0=P_1\).

\emph{Importance-resampling consistency.}
Let \(U_1,\ldots,U_K\) be independent candidates sampled from \(P_1\).
Given this proposal pool, the resampling procedure in \Cref{eq:methods_resampling} induces a weighted empirical distribution \(\widehat Q_{\beta,K}\).
For any completion \(v\in\mathcal U\), its probability of being selected is
\[
    \widehat Q_{\beta,K}(v)
    =
    \frac{
        \frac{1}{K}\sum_{k=1}^{K}
        \mathbf 1\{U_k=v\}\,e^{\beta R(U_k)}
    }{
        \frac{1}{K}\sum_{k=1}^{K}e^{\beta R(U_k)}
    }.
\]
The numerator accumulates the weights of all proposals equal to \(v\), including repeated occurrences, while the denominator normalizes the weights across the pool.

By the strong law of large numbers, the numerator converges almost surely to \(P_1(v)e^{\beta R(v)}\), while the denominator converges almost surely to \(Z_\beta>0\).
Their ratio therefore converges to
\[
    \widehat Q_{\beta,K}(v)
    \longrightarrow
    \frac{P_1(v)e^{\beta R(v)}}{Z_\beta}
    =
    Q_\beta(v)
    \qquad\text{almost surely}.
\]
Since the candidate space \(\mathcal U\) is finite, this convergence holds simultaneously for every completion.
Consequently,
\[
    \frac{1}{2}
    \sum_{v\in\mathcal U}
    \abs{\widehat Q_{\beta,K}(v)-Q_\beta(v)}
    \longrightarrow 0
    \qquad\text{almost surely},
\]
which establishes convergence in total variation.
\end{proof}

This guarantee is asymptotic and provides no finite-\(K\) error bound.
It also does not imply that higher example reward corresponds to greater correctness.

\subsection{Extension to Unknown Lengths}
\label{app:proofs:joint_length}

With unknown target lengths, each complete candidate consists of a length allocation \(L\) and an assignment \(u\in V^{T_L}\) to its expanded target positions.
Let \(\mathcal L\) denote the finite support of the length allocator \(\pi_\psi(\cdot\mid\widetilde z)\).
For every \(L\in\mathcal L\), \Cref{lem:background_any_order} ensures that the conditional infill distributions \(P(\cdot\mid z_L)\) and \(P(\cdot\mid z_L\setminus E_L)\) are normalized over \(V^{T_L}\).
Since the length allocator is also normalized, both joint distributions in \Cref{eq:methods_joint_lengths} are normalized.
Assuming their conditional infill probabilities are positive, the joint distributions satisfy the assumptions of \Cref{prop:methods_sampling} on their common support.
Moreover, the shared allocator cancels from the likelihood ratio, giving
\[
    R(L,u)
    =
    \log
    \frac{P(u\mid z_L)}
         {P(u\mid z_L\setminus E_L)}.
\]
Thus, reward guidance and importance resampling extend directly to joint length-content candidates.

Although the allocator cancels from each candidate's reward, guidance can still change the distribution over selected lengths.
To see this, define the length-specific normalizing constant
\[
    Z_\beta(L)
    :=
    \expval{u\sim P(\cdot\mid z_L)}
    e^{\beta R(L,u)}.
\]
The overall normalizer and the marginal guided length distribution then satisfy
\[
\begin{aligned}
    Z_\beta
    &=
    \sum_{L\in\mathcal L}
    \pi_\psi(L\mid\widetilde z)\,Z_\beta(L), \\
    Q_\beta(L)
    &=
    \sum_{u\in V^{T_L}}Q_\beta(L,u)
    =
    \frac{
        \pi_\psi(L\mid\widetilde z)\,Z_\beta(L)
    }{
        Z_\beta
    }.
\end{aligned}
\]
Hence, length allocations with larger \(Z_\beta(L)\) receive greater relative probability under guidance, even though the allocator itself is unchanged.
The argument does not require \(\pi_\psi\) to match the data length distribution or to factorize across target sites; oracle lengths correspond to a single supported allocation.

\providecommand{\panelcodefont}{\fontsize{8.3pt}{9.9pt}\selectfont}

\DefineVerbatimEnvironment{PromptBlock}{Verbatim}{%
    breaklines=true,
    breakanywhere=false,
    breakautoindent=true,
    breaksymbolleft={\tiny\ensuremath{\hookrightarrow}},
    breaksymbolright={},
    breaksymbolsepleft=0.4em,
    fontsize=\panelcodefont,
    highlightcolor=black!9,
    vspace=0pt,
    frame=none,
    xleftmargin=0pt,
    xrightmargin=0pt
}
\DefineVerbatimEnvironment{ReferenceBlock}{Verbatim}{%
    breaklines=true,
    breakanywhere=false,
    breakautoindent=true,
    breaksymbolleft={\tiny\ensuremath{\hookrightarrow}},
    breaksymbolright={},
    breaksymbolsepleft=0.4em,
    fontsize=\panelcodefont,
    vspace=0pt,
    frame=none,
    xleftmargin=0pt,
    xrightmargin=0pt
}

\tcbset{
    appendixpanel/.style={
        enhanced,
        colback=white,
        colframe=black!55,
        boxrule=0.4pt,
        arc=1pt,
        boxsep=0pt,
        left=5pt, right=5pt, top=4pt, bottom=4pt,
        before skip=6pt plus 2pt minus 1pt,
        after skip=8pt plus 2pt minus 1pt,
        fonttitle=\footnotesize\bfseries,
        coltitle=black,
        colbacktitle=black!8,
        titlerule=0.4pt,
        toptitle=2.2pt, bottomtitle=2.2pt,
        lefttitle=4pt, righttitle=4pt,
        parbox=false
    },
    promptpanel/.style={
        appendixpanel,
        breakable,
        lines before break=4,
        pad at break*=4pt
    },
    referencepanel/.style={
        enhanced,
        colback=black!6,
        colframe=black!45,
        boxrule=0pt,
        leftrule=1.4pt,
        sharp corners,
        boxsep=0pt,
        left=5pt, right=3pt, top=2.5pt, bottom=2.5pt,
        before skip=5pt,
        after skip=0pt,
        parbox=false
    }
}
\newtcolorbox{promptpanel}[2]{promptpanel, title={#1\hfill{\normalfont\footnotesize #2}}}
\newtcolorbox{referencepanel}{referencepanel}
\newcommand{\attachreference}{\par\unskip\unskip\unskip\nopagebreak}
\newcommand{\referenceoutputtext}{Reference (hidden from model)}
\newcommand{\referenceoutputlabel}{%
    \par\nointerlineskip
    \sbox0{\panelcodefont\strut}%
    \vbox to 0pt{\vss\hbox to\linewidth{\hfill
        \raisebox{-\ht0}[0pt][0pt]{\footnotesize\itshape\color{black!60}\referenceoutputtext}}}%
    \nointerlineskip
}

\section{Few-Shot Prompting Experiments}
\label{app:additional_experiments}

This appendix documents the experiments of \Cref{sec:experiments} and provides additional analyses of demonstration conditioning, example guidance, and length allocation.

\subsection{Experimental Setup}
\label{app:experiment_setup}
\label{app:benchmarks}

\paragraph{Benchmarks.}
We study four types of few-shot transformations: classification, list transformation, artificial-language transformation, and information extraction.
For SST-2~\citep{socher2013recursive,wang2018glue}, sentiment labels are replaced by two resampled single-token capital-letter codes; the first two demonstrations cover both classes, and answers are graded by exact code match.
List Functions--Easy~\citep{srivastava2022beyond,rule2020child} uses 85 unconditional functions adapted from BIG-bench List Functions; the transformation rule is hidden, and parsed integer lists are matched exactly.
SALT repetition~\citep{zheng2025curse,zheng2025logidynamics} uses categories 5--6 with a separate dictionary for each example to translate words and repeat each verb or adjective; generated pairs pass an executable-rule check, and outputs are matched exactly.
Movie Director~\citep{liu2012conversational,zhao2021calibrate} extracts the \texttt{Director} field from audited MIT Movie queries without naming the field in the instruction; labels are fixed, and outputs are graded by exact string match.

\paragraph{Experimental studies.}
We use three separately sampled studies with frozen models: a single-target demonstration sweep across four benchmarks and four model interfaces (\(J=1\), \(K=64\)); a List Functions--Easy \(D\times J\) grid across three DLMs (\(K=16\)); and a cross-benchmark target-count sweep across four benchmarks and three DLMs (\(D=4\), \(K=64\)).
Each \(K\) specifies the number of samples per canvas, and the studies are analyzed separately.

\paragraph{Prompt families.}
For single-target prompting, we draw 16 demonstrations followed by one target, then retain the first \(D\in\{1,2,4,8,16\}\) demonstrations to construct five canvases with the same target.
For multi-target prompting, we draw 16 demonstrations and eight targets.
We first randomly interleave four demonstrations and eight targets (12 items), then extend this sequence with the remaining 12 demonstrations (24 items), preserving the relative order of the original items.
Each \((D,J)\) canvas retains the first \(D\) demonstrations and first \(J\) targets in that order.
The canvases from one draw form a \emph{prompt family}; all evaluation examples come from sources disjoint from development.

\paragraph{Model interfaces and decoding.}
LLaDA-1.5~\citep{zhu2026llada} and dQwen3.5-9B-Base~\citep{xue2026dqwen} directly infill masked canvases.
Nemotron-Labs-Diffusion-8B~\citep{fu2026nemotron} was trained with limited bidirectional diffusion blocks, so we use a fixed-length causal suffix for the single-target sweep and a causal prefix with a bidirectional tail for both target-count studies, including \(J=1\).
Qwen3-8B~\citep{yang2025qwen3} provides a free-length causal-continuation comparison for single targets only.
DLMs reveal one target token at a time in ascending canvas order for both generation and likelihood scoring, using oracle lengths except in \Cref{sec:experiments:lengths}.

\paragraph{Sampling, scoring, and statistics.}
We sample at token temperature one, retain duplicate infills, and reuse each sample pool for unguided and guidance-weighted evaluation.
The \(P_1\) scores keep demonstration outputs visible, while \(P_0\) masks them for DLMs or removes them from causal prompts; sampled contents, target lengths, and reveal orders are unchanged.
Malformed outputs count as incorrect, and we average \(A_{\mathrm{all}}\) and \(A_{\mathrm{avg}}\) across prompt families with unadjusted 95\% paired family-bootstrap intervals from 10,000 resamples.
Main guidance comparisons use \(\beta=0.5\), while length-allocation experiments use \(\beta=0\).

\paragraph{Example prompts.}
The following show one \(D=2\), \(J=1\) logical prompt per benchmark, omitting model-specific chat tokens.
\texttt{[MASKED TARGET]} denotes an oracle-length mask span for DLMs, while causal generation starts after the target's \texttt{Output:}; reference answers are shown only to the reader.

\begin{promptpanel}{SST-2 with demonstration-defined codes}{$D{=}2$, $J{=}1$}
\begin{PromptBlock}[highlightlines={1,4}]
User:
Each input is followed by its output code. Produce the output code for the final input.

Assistant:
Input: an uninspired preachy and clichéd war
Output: F

Input: miraculous as its dreamworks makers
Output: E

Input: a skillful filmmaker can impart a message without bludgeoning the audience over the head
Output: [MASKED TARGET]
\end{PromptBlock}
\attachreference
\begin{referencepanel}
\referenceoutputlabel
\begin{ReferenceBlock}
E
\end{ReferenceBlock}
\end{referencepanel}
\end{promptpanel}

\begin{promptpanel}{List Functions--Easy}{$D{=}2$, $J{=}1$}
\begin{PromptBlock}[highlightlines={1,4}]
User:
Apply the transformation illustrated by the examples.

Assistant:
Input: [7, 8, 4, 9, 6, 0, 5]
Output: [5, 0, 6, 9, 4, 8, 7]

Input: [9, 3, 5, 1, 6, 2, 4]
Output: [4, 2, 6, 1, 5, 3, 9]

Input: [2, 7, 8]
Output: [MASKED TARGET]
\end{PromptBlock}
\attachreference
\begin{referencepanel}
\referenceoutputlabel
\begin{ReferenceBlock}
[8, 7, 2]
\end{ReferenceBlock}
\end{referencepanel}
\end{promptpanel}

\begin{promptpanel}{SALT repetition: repeat every verb}{$D{=}2$, $J{=}1$}
\begin{PromptBlock}[highlightlines={1,6}]
User:
Below are several examples that convert english sentence into an output sequence based on specific rules.
Each word in the input sequence either corresponds to a translated word in the output,
or indicates a syntactic rule (e.g., repeating or reordering semantic units) for forming the output sequence.

Assistant:
Input: flower bloom bee
Output: ngq mqh mqh owx

Input: bee pollinate seed colorful
Output: owx zkl zkl qqn yil

Input: pollinate colorful seed
Output: [MASKED TARGET]
\end{PromptBlock}
\attachreference
\begin{referencepanel}
\referenceoutputlabel
\begin{ReferenceBlock}
zkl zkl yil qqn
\end{ReferenceBlock}
\end{referencepanel}
\end{promptpanel}

\begin{promptpanel}{Movie Director extraction}{$D{=}2$, $J{=}1$}
\begin{PromptBlock}[highlightlines={1,5}]
User:
Extract the same kind of information as in the examples.
Copy the answer exactly from the input. Return only the answer.

Assistant:
Input: barbara stanwyck played a wicked scheming woman duping an insurance company in cahoots with fred macmurray in this classic 1944 film noir directed by billy wilder
Output: billy wilder

Input: what s the woody allen movie where owen wilson finds himself going back in time and meeting famous people
Output: woody allen

Input: what 2005 ron howard biopic tells the story of james braddock a washed up boxer who came back to become a champion
Output: [MASKED TARGET]
\end{PromptBlock}
\attachreference
\begin{referencepanel}
\referenceoutputlabel
\begin{ReferenceBlock}
ron howard
\end{ReferenceBlock}
\end{referencepanel}
\end{promptpanel}

\subsection{Single-Target Experiments}
\label{app:single_target_details}
\label{app:infilling_details}

To examine why additional demonstrations improve unguided accuracy in \Cref{sec:experiments:completion}, we measure how their presence changes the likelihood of the reference answer.
For each prompt family, we retain the same target as \(D\) grows and evaluate the reference completion \(u^\star\) under both demonstration contexts.
We report its example reward per target token:
\[
    \frac{R(u^\star)}{|T|}
    = \frac{\log P_1(u^\star)-\log P_0(u^\star)}{|T|}.
\]
This is equivalently the reduction in reference-answer negative log likelihood per token when demonstration outputs become visible.
The reward generally increases with \(D\), indicating that additional demonstrations provide greater relative support for the reference answer.
Because \(P_0\) masks demonstration outputs for DLMs but removes them from causal prompts, reward magnitudes are not directly comparable across these interfaces.

Greater support for a reference answer does not, however, guarantee a larger gain from guidance.
At \(D=1\), SST-2 exposes only one class code, while a single List Functions or SALT demonstration may be compatible with several rules.
At larger \(D\), unguided accuracy can approach saturation, leaving less room for additional improvement.
Moreover, guidance depends on how the reward ranks the sampled alternatives, not only on its value for the reference answer.

\subsection{Multi-Target Experiments}
\label{app:mj_grid}
\label{app:multi_target_details}
\label{app:multitarget_m4}

\Cref{sec:experiments:multitarget} shows that the effects of example guidance vary with the number of jointly infilled targets.
We investigate this behavior with two additional analyses: a reference-likelihood diagnostic that follows successive targets, and a separate cross-benchmark target-count study.
The latter uses \(D=4\) and \(K=64\) samples per canvas, rather than the \(K=16\) pools of the main List Functions--Easy \(D\times J\) grid.

\paragraph{Reference-likelihood diagnostic.}
Previously revealed targets may supply information that would otherwise come from the demonstrations.
To examine this possibility, we evaluate the reference tokens in reveal order with \(D=4\) and \(J=8\), making each preceding reference target visible under both \(P_1\) and \(P_0\) while keeping demonstration outputs masked under \(P_0\).
In \Cref{fig:multitarget_mechanism}, the gap between subsequent reference-token negative log likelihoods generally narrows as more targets are revealed.

\ifdefined\arxivbuild\begin{figure}[t]\else\begin{figure*}[t]\fi
    \centering
    \includegraphics[width=\textwidth]{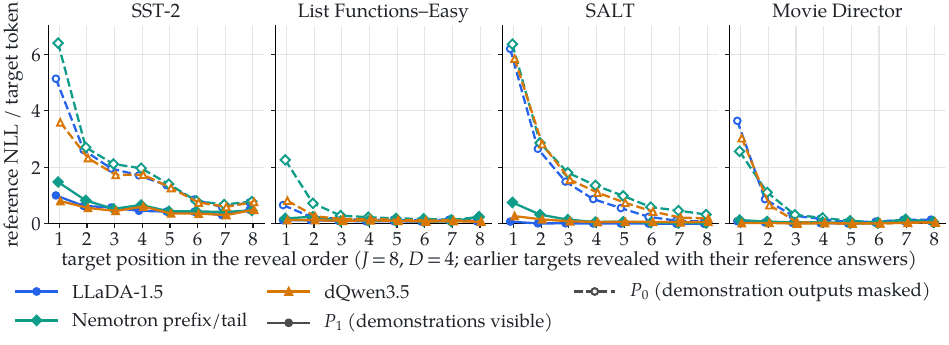}
    \caption{\textbf{Reference-token NLL under \(P_1\) and \(P_0\) by reveal position (cross-benchmark study, \(J=8\)).}
    Previously revealed reference targets are visible in both contexts, and the gap between their subsequent reference-token scores generally narrows.
    This diagnostic follows reference assignments, not sampled decoding trajectories.}
    \label{fig:multitarget_mechanism}
\ifdefined\arxivbuild\end{figure}\else\end{figure*}\fi

\paragraph{Cross-benchmark target-count study.}
The main \(D\times J\) study uses only List Functions--Easy, so we also examine whether its guidance trends hold across other tasks.
In this separate experiment, three DLMs are evaluated on four benchmarks at \(D=4\), \(J\in\{1,2,4,8\}\), and \(K=64\) (\Cref{fig:app_multitarget_m4}).
At \(J=8\), guidance at \(\beta=0.5\) improves \(A_{\mathrm{all}}\) on SST-2, decreases it on Movie Director, is nearly neutral on List Functions--Easy, and varies by model on SALT.
Thus, the reduced guidance benefits observed on List Functions--Easy do not hold uniformly across tasks.

\ifdefined\arxivbuild\begin{figure}[t]\else\begin{figure*}[t]\fi
    \centering
    \includegraphics[width=\textwidth]{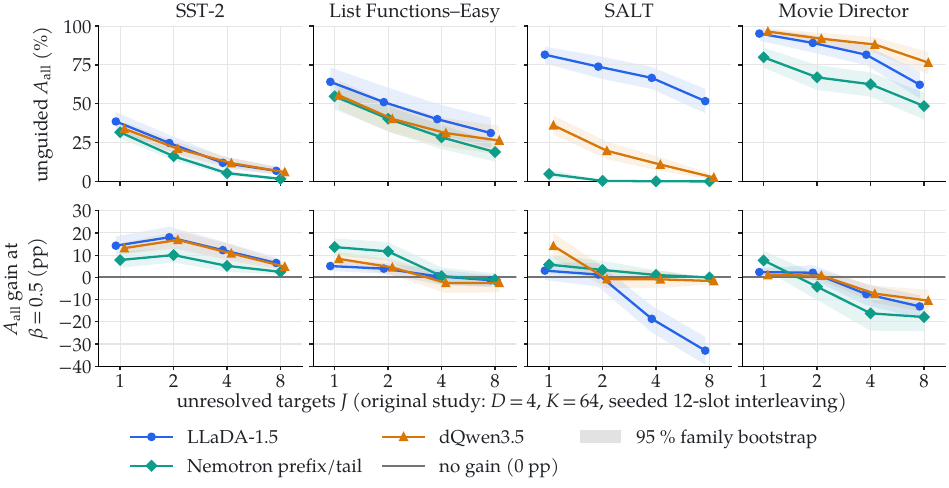}
    \caption{\textbf{Cross-benchmark target-count study at the common guidance strength.}
    Top: unguided all-target accuracy \(A_{\mathrm{all}}\); bottom: its change under example guidance at \(\beta=0.5\), which varies by task at \(J=8\).
    \(D=4\), \(K=64\) samples per canvas, oracle lengths, three DLMs, and 64 prompt families per cell; bands are unadjusted 95\% family-bootstrap intervals.}
    \label{fig:app_multitarget_m4}
\ifdefined\arxivbuild\end{figure}\else\end{figure*}\fi

\subsection{Length Allocation}
\label{app:length_details}

\paragraph{Predictor architecture.}
To evaluate infilling without oracle target lengths, we train length predictors using frozen DLM representations, following \citet{xu2026predict}.
Each compact prompt \(\widetilde z\) contains one placeholder per target, whose hidden state before the final transformer block is passed through a trainable copy of that block and its final normalization.
For LLaDA-1.5 this is block 31 of 32 with \texttt{ln\_f}; for Nemotron it is layer 33 of 34 with \texttt{encoder.norm}; and for dQwen3.5 it is layer 31 of 32 with \texttt{model.norm}.
A categorical head predicts lengths in \(\{1,\ldots,256\}\) independently across sites, conditional on the compact prompt.
Only the length predictor is trained, using prompt-normalized cross-entropy on the reference token lengths under each backbone's tokenizer; the original DLM remains frozen and jointly infills the allocated targets.

\paragraph{Training protocol.}
Training draws from a 70/30 mixture of structured few-shot prompts (\(100{,}000\) from MIT-Movie, ATIS, List Functions, and SALT) and Java editing episodes (\(22{,}128\) from 28 repositories, split 21/4/3 by repository).
We exclude evaluation source units, all 85 List Functions--Easy rules, SALT categories 5--6, the MIT-Movie \texttt{Director} field, and all 18 SWE-Refactor repositories.
The Java predictor used in \Cref{sec:case_studies} is trained separately (\Cref{app:case_study_details:zero_demo}).
For optimization, we use AdamW with weight decay \(0.01\), learning rates \(3\times10^{-5}\) for the copied block and normalization and \(3\times10^{-4}\) for the categorical head.
We train for 1,563 steps with 64 draws per update, 5\% warm-up, cosine decay, evaluating the final checkpoint.

\paragraph{Length fidelity.}
For a target with reference length \(L_j^\star\), we measure the predictor's probability mass within \(r\) tokens of that length:
\[
    F_r(\widetilde z,j)
    := \sum_{\ell:\,|\ell-L_j^\star|\leq r}
       \pi_{\psi,j}(\ell\mid\widetilde z).
\]
We average over evaluation targets and prompt families; \(F_0\) measures probability assigned to the reference length, not the accuracy of the most likely prediction.
Across backbones and target counts, \(F_0=40\)--\(46\%\) and \(F_2=53\)--\(61\%\), versus approximately 27\% exact-length mass from a prior based on four demonstration-output lengths.

\paragraph{Downstream evaluation.}
On List Functions--Easy, we fix \(D=4\), vary \(J\in\{1,2,4,8\}\), and compare oracle lengths with sampled predictor lengths using the same content decoder and \(K=64/J\) samples per canvas at \(\beta=0\).
Predicted lengths retain 35--48\% of oracle-length per-target accuracy (\Cref{fig:predicted_length_infilling}, top).
To assess sensitivity separately from the predictor's own errors, we impose \(L_j=L_j^\star+\Delta\) for \(J\in\{1,2\}\) (\Cref{fig:predicted_length_infilling}, bottom), skipping a prompt family at any \(\Delta\) for which some \(L_j^\star+\Delta<1\); these forced offsets do not represent lengths sampled from the predictor.

\ifdefined\arxivbuild
\subsection{Guidance-Strength Sensitivity}
\label{app:guidance_strength}

We examine whether increasing the guidance strength \(\beta\) improves accuracy or instead concentrates resampling weights on fewer samples.
Using stored cross-benchmark List Functions--Easy pools with \(D=4\), we compare \(J\in\{1,4\}\) at \(K=64\) over 64 prompt families per cell, reweighting the same samples for \(\beta\in\{0,0.25,0.5,1,2\}\).
To describe weight concentration, define the entropy and entropy-based effective fraction of the normalized sample weights \(\widetilde w_k(\beta)\):
\[
    H_\beta := -\sum_{k=1}^{K}\widetilde w_k(\beta)\log\widetilde w_k(\beta),
    \qquad
    \rho_{\mathrm{eff}}(\beta) := \frac{\exp(H_\beta)}{K}.
\]
The fraction \(\rho_{\mathrm{eff}}\) measures concentration over sampled draws, counting duplicates separately, rather than diversity among distinct completions.
\Cref{fig:app_guidance_strength} compares its mean against gains in \(A_{\mathrm{all}}\) relative to \(\beta=0\).

\begin{figure}[t]
    \centering
    \includegraphics[width=\textwidth]{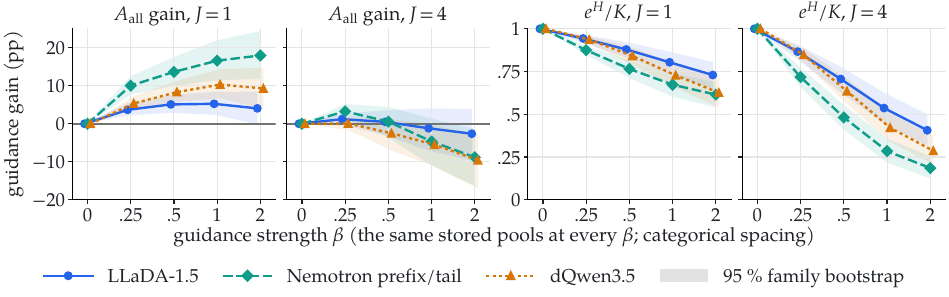}
    \caption{\textbf{Guidance-strength sensitivity on List Functions--Easy (cross-benchmark study).}
    Change in \(A_{\mathrm{all}}\) from \(\beta=0\) at \(J=1\) and \(J=4\) (left), and mean \(\rho_{\mathrm{eff}}\), the entropy-based effective fraction of sampled draws (right).
    Each strength reweights the same \(D=4\), \(K=64\) sample pools, with duplicates retained and oracle lengths; Nemotron uses a causal prefix and bidirectional tail even at \(J=1\).
    Bands are unadjusted 95\% paired family-bootstrap intervals from 10,000 resamples.}
    \label{fig:app_guidance_strength}
\end{figure}

For a single target, every tested \(\beta>0\) improves mean \(A_{\mathrm{all}}\) for all three DLMs (intervals exclude zero except for LLaDA at \(\beta=2\)).
Gains peak near \(\beta=1\) for LLaDA and dQwen3.5 (+5.2 and +10.3 percentage points), while Nemotron reaches +18.0 points at \(\beta=2\).
With four targets, however, gains reach at most +3.3 points at \(\beta=0.25\) and then diminish or reverse, reaching \(-2.6\) to \(-9.6\) points at \(\beta=2\) (intervals exclude zero for Nemotron and dQwen3.5).

Weight concentration alone does not appear to explain these accuracy trends.
When accuracy deteriorates at \(J=4\), \(\rho_{\mathrm{eff}}\) remains 0.48--0.71, whereas single-target accuracy still improves at comparable fractions (0.62--0.73 at \(\beta=2\)).
Four-target rewards also span roughly four times as many tokens and are more dispersed, so this comparison does not isolate target count as the cause.

\fi

\subsection{Likelihood Sharpening}
\label{app:power_details}

Some gains from example guidance could reflect sharpening of the demonstration-conditioned likelihood \(P_1\), rather than the additional information in the \(P_1/P_0\) reward.
To separate these effects, we compare example guidance with likelihood sharpening that uses the same total exponent on \(P_1\).
At \(\beta=0.5\), the guided distribution, the sharpening distribution, and the sample weights for sharpening are
\[
    Q_{0.5}(u) \propto P_1(u)^{3/2}P_0(u)^{-1/2},
    \qquad
    Q_{\mathrm{sharp}}(u) \propto P_1(u)^{3/2},
    \qquad
    w_{\mathrm{sharp}}(u) \propto P_1(u)^{1/2}.
\]
Both methods reweight the same samples without additional sampling or tuning.
The comparison matches the \(P_1\) exponent, but not necessarily weight concentration or a separately tuned decoding temperature.

\ifdefined\arxivbuild\begin{figure}[t]\else\begin{figure*}[t]\fi
    \centering
    \includegraphics[width=\textwidth]{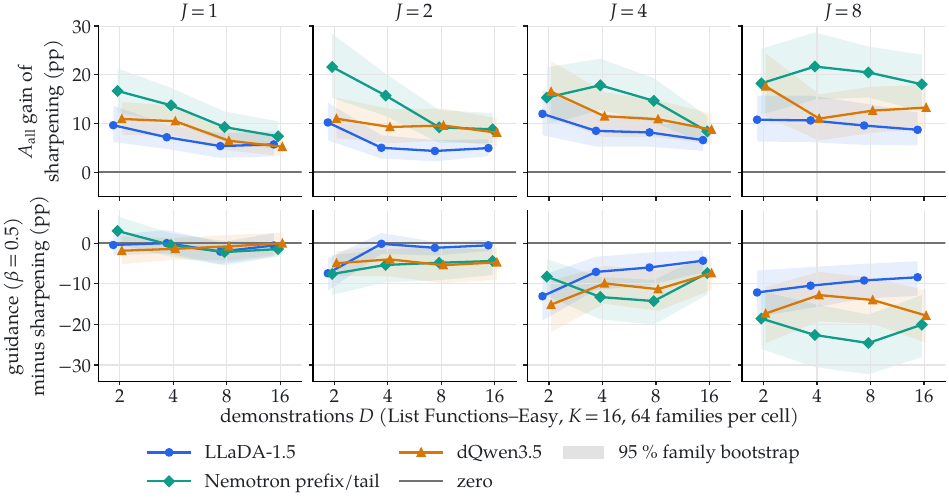}
    \caption{\textbf{Likelihood sharpening with the same total \(P_1\) exponent in the \(D\times J\) grid.}
    Top: gain in \(A_{\mathrm{all}}\) from sharpening alone (sample weights \(\propto P_1^{1/2}\), resulting distribution \(\propto P_1^{3/2}\)); bottom: example guidance minus sharpening (negative values favor sharpening).
    Both reweight the \(K=16\) sample pools of \Cref{fig:multitarget_evidence}; bands are unadjusted 95\% family-bootstrap intervals.}
    \label{fig:app_mj_power}
\ifdefined\arxivbuild\end{figure}\else\end{figure*}\fi

Likelihood sharpening explains much of the improvement from guidance, although the example-dependent reward provides an additional gain in one setting.
For LLaDA with a single target at \(D=4\), guidance improves over sharpening by +3.6 percentage points ([+0.4, +6.9]) on SST-2; no other single-target task shows a positive residual.
In the List Functions--Easy \(D\times J\) grid, sharpening improves \(A_{\mathrm{all}}\) for every \(J\) and outperforms guidance in every cell with \(J\geq4\) (\Cref{fig:app_mj_power}).
These comparisons distinguish the gains attributable to likelihood sharpening from the additional contribution of the example-dependent reward.

\providecommand{\panelcodefont}{\fontsize{8.3pt}{9.9pt}\selectfont}
\DefineVerbatimEnvironment{CaseCanvasCode}{Verbatim}{%
    breaklines=true,breakanywhere=false,breakautoindent=true,breakafter={.},
    breaksymbolleft={\tiny\ensuremath{\hookrightarrow}},breaksymbolright={},breakaftersymbolpre={},breakaftersymbolpost={},
    breaksymbolsepleft=0.4em,fontsize=\panelcodefont,
    vspace=0pt,frame=none,xleftmargin=0pt,xrightmargin=0pt
}
\DefineVerbatimEnvironment{CaseSideCode}{Verbatim}{%
    breaklines=true,breakanywhere=false,breakautoindent=true,breakafter={.},
    breaksymbolleft={\tiny\ensuremath{\hookrightarrow}},breaksymbolright={},breakaftersymbolpre={},breakaftersymbolpost={},
    breaksymbolsepleft=0.4em,fontsize=\panelcodefont,
    vspace=0pt,frame=none,xleftmargin=0pt,xrightmargin=0pt
}

\newtcolorbox{codepanel}[2][]{appendixpanel, breakable, title={#2}, #1}
\newcommand{\codepanelnote}[1]{\noindent{\footnotesize\itshape #1}\par\vspace{2pt}}
\section{Software Editing Experiments}
\label{app:case_study_details}

We describe the Java refactoring benchmark and evaluation, then analyze demonstration conditioning, length allocation, and example guidance.

\subsection{Benchmark Construction and Evaluation Setup}
\label{app:case_study_details:construction}
\label{app:case_study_details:source}
\label{app:case_study_details:eligibility}

We derive multi-site editing tasks from \emph{Extract Method} commits in SWE-Refactor~\citep{xu2026swerefactor}, which provide before-and-after snapshots and extracted helpers.
For each (repository, commit, file, helper) group, we identify updated callers and represent each minimal changed region as a \texttt{<old>}/\texttt{<new>} pair.
A seeded file-order rule assigns \(D\) accepted edits as demonstrations and \(J\) remaining sites as targets.
Eligibility rules fixed before inspecting outcomes exclude edits requiring identifiers or literals absent from the prompt or changing asserted values.
We also limit replacements to 80 LLaDA tokens and enclosing methods to 200 lines, retain one helper per commit and file, and collapse mirrored commits.
The resulting benchmark contains 127 tasks from 50 commits across 14 repositories, with \((D,J)\in\{1,2\}^2\) and at most one task per group and configuration.
Because eligible commits differ across configurations, comparisons across \((D,J)\) in \Cref{tab:case_study_results_by_cell} do not isolate the effect of demonstration or target count; results are aggregated and bootstrapped by commit.

\begin{figure*}[t]
\begin{tcolorbox}[appendixpanel, before skip=0pt, after skip=0pt,
    title={One editing canvas\hfill{\normalfont\footnotesize Checkstyle \texttt{21c8e04bc9}, \(D{=}1\), \(J{=}2\)}}]
\noindent\begin{minipage}[t]{0.485\linewidth}
\noindent{\footnotesize\textbf{Provided context}}\par\vspace{0.15em}
\noindent{\footnotesize\textit{Helper definition}}\par\vspace{0.15em}
\begin{CaseCanvasCode}
private static boolean isChildNodeTextMatches(
        DetailNode ast, Pattern pattern) {
    return getFirstChildOfType(ast,
            JavadocTokenTypes.TEXT, 0)
            .map(DetailNode::getText)
            .map(pattern::matcher)
            .map(Matcher::matches)
            .orElse(false);
}
\end{CaseCanvasCode}
\vspace{0.3em}
\noindent{\footnotesize\textit{Accepted edit (demonstration)}}\par\vspace{0.15em}
\begin{CaseCanvasCode}
private static boolean isExamplesText(
        DetailNode ast) {
<old>
    final Optional<DetailNode> resultNode =
        getFirstChildOfType(ast,
            JavadocTokenTypes.TEXT, 0);
    return resultNode.isPresent()
        && EXAMPLES_TAG.matcher(
            resultNode.get().getText()).matches();
</old>
<new>
    return isChildNodeTextMatches(
        ast, EXAMPLES_TAG);
</new>
}
\end{CaseCanvasCode}
\end{minipage}\hfill
\begin{minipage}[t]{0.485\linewidth}
\noindent{\footnotesize\textbf{Sites to complete}}\par\vspace{0.15em}
\noindent{\footnotesize\textit{Target 1}}\par\vspace{0.15em}
\begin{CaseCanvasCode}
private static boolean isPropertyList(
        DetailNode nodeLi) {
<old>
    final Optional<DetailNode>
        firstTextChildToken = getFirstChildOfType(nodeLi,
            JavadocTokenTypes.TEXT, 0);
    return firstTextChildToken.isPresent()
        && PROPERTY_TAG.matcher(
            firstTextChildToken.get()
                .getText()).matches();
</old>
<new>[MASKED TARGET]</new>
}
\end{CaseCanvasCode}
\vspace{0.3em}
\noindent{\footnotesize\textit{Target 2}}\par\vspace{0.15em}
\begin{CaseCanvasCode}
private static boolean isViolationMessagesText(
        DetailNode nodeParagraph) {
<old>
    final Optional<DetailNode> resultNode =
        getFirstChildOfType(nodeParagraph,
            JavadocTokenTypes.TEXT, 0);
    return resultNode.isPresent()
        && VIOLATION_MESSAGES_TAG.matcher(
            resultNode.get().getText()).matches();
</old>
<new>[MASKED TARGET]</new>
}
\end{CaseCanvasCode}
\end{minipage}
\begin{referencepanel}
\referenceoutputlabel
\begin{ReferenceBlock}
Target 1: return isChildNodeTextMatches(nodeLi, PROPERTY_TAG);
Target 2: return isChildNodeTextMatches(nodeParagraph, VIOLATION_MESSAGES_TAG);
\end{ReferenceBlock}
\end{referencepanel}
\end{tcolorbox}
\caption{\textbf{One editing canvas from a Checkstyle refactoring.}
An accepted edit introduces a helper call, and two masked targets must adapt it to their local arguments.
Both columns belong to the same prompt; hook arrows indicate visual line wraps.}
\label{fig:case_study_checkstyle_canvas}
\end{figure*}

The figure is an excerpt: actual prompts include instructions, chat framing, and additional context from the accepted after-file, with target edits reverted and context cropped to the helper and enclosing methods.
For \(P_0\), only changed-line tokens of demonstration outputs are masked.

\paragraph{Evaluation procedure.}
Frozen LLaDA-1.5~\citep{zhu2026llada} generates \(K=16\) rewrites per task under each length setting; unguided sampling and example guidance use the same stored samples.
We use Claude Sonnet 5.5 (\texttt{claude-sonnet-5-5}) as a Claude Code sub-agent, with a fixed judging rubric and fresh context per call.
Each judging packet contains the helper and its Javadoc, accepted demonstrations, and target methods with candidate rewrites inserted, but excludes developer references, generation settings, length allocations, and model scores.
The judge assesses whether the edits adapt receivers, arguments, and control flow while plausibly preserving behavior; equivalent rewrites may pass, while incomplete or erroneous edits fail.
Joint success requires every target to be judged correct; uncertain judgments count as failures, and ill-posed packets remain in the evaluation.
Duplicate candidate texts are judged once per task, but duplicate draws remain in the sample pools; we report unadjusted 95\% paired commit-bootstrap intervals.
We do not compile or execute the generated code.

\paragraph{Judge checks and limitations.}
The 4,064 generated draws yielded 2,089 distinct judging packets after deduplication within tasks, including across length settings.
The judge classified all 46 blinded validation cases as expected, and ten repeated judgments agreed, although these checks do not replace execution-based evaluation.
Its inherited agent context contained brief summaries of earlier Section~5 results, but no packet-specific labels; judge sampling parameters were not explicitly set or recorded.
All 50 source commits were examined during development, and some prompts contain target-reference text elsewhere in the helper or demonstrations, so this study is not a held-out evaluation.

\subsection{Demonstration Conditioning and Length Allocation}
\label{app:case_study_details:zero_demo}
\label{app:case_study_details:length}

\paragraph{Demonstration ablation.}
To measure the effect of demonstration conditioning, we compare \(D=0\) and \(D=1\) on the same 80 editing tasks (52 with \(J=1\), 28 with \(J=2\)) in a post-hoc ablation.
Both conditions share the helper, target sites, oracle lengths, chat framing, \(K=16\), and unguided sampling (\(\beta=0\)).
They use the instruction: ``Complete each target \texttt{<new>} span by applying the intended refactoring using the helper shown in the source context; use any accepted before/after edits if provided.''
Only demonstration visibility changes during generation; the judge sees the accepted edit in both conditions.

Adding one demonstration improves joint LLM-judged success by 14.7 percentage points for \(J=1\) (\([+5.8,+24.5]\)) and 17.5 points for \(J=2\) (\([+4.7,+32.5]\)).
The estimated gain is smaller on commits outside the narrower LLaDA development subset (\(+5.3\) points, \([-1.0,+13.9]\)).
The ablation was designed after inspecting the main results but fixed before collecting its own outcomes; its instruction differs from the main experiment, so absolute success rates are not directly comparable.

\paragraph{Java length predictor.}
We train a Java-specific length predictor using the LLaDA-1.5 architecture and objective of \Cref{app:length_details}, with edits from 21 training and four development repositories in the 28-repository corpus; the remaining three repositories and all SWE-Refactor repositories are excluded.
Training prompts contain an optional helper declaration, \(D\) accepted \texttt{<old>}/\texttt{<new>} demonstrations, and \(J\) targets with visible \texttt{<old>} spans and one placeholder per missing \texttt{<new>} span.
We use seeded, disjoint site assignments, retain up to 20 context lines before and 12 after each edit, and exclude empty replacements, lengths above 256 tokens, and prompts exceeding 2,048 tokens.
Reference lengths are withheld from the inputs, with prompt-invariance checks to detect leakage.
This yields helper-only and all-multi-site corpora with 1,887/370 and 13,543/3,598 training/development prompts, respectively, where the latter includes the former.
Training alternates between the corpora equally for 313 updates (batch size 64, 16 warm-up updates, cosine decay), with other settings following \Cref{app:length_details}.
We also train an alternative on the all-multi-site corpus alone and select by exact-length probability mass \(F_0\) on 488 helper-only development targets (0.0664 versus 0.0660), although within-two-token mass \(F_2\) would favor the alternative (0.317 versus 0.328); SWE-Refactor outcomes are not used for selection.

\paragraph{Length evaluation.}
At evaluation, the selected predictor runs once per task and independently samples target lengths for each of \(K=16\) rewrites, without temperature adjustment or truncation.
Oracle lengths equal the developer-reference token counts; 3.3\% of predicted per-target lengths exceed the benchmark's 80-token reference-length eligibility limit.
We compare length fidelity against the empirical distribution of 2,467 replacement lengths in the helper-only training corpus, distinct from the demonstration-length prior in \Cref{sec:experiments:lengths}.
Across 80 unique target sites, the predictor assigns 26.7\% probability mass within two tokens of the reference length (\([20.1,33.8]\)), compared with 15.2\% for this prior; exact-length mass is 6.8\% versus 3.1\%.
However, only 6.2\% of learned-length joint samples allocate every target exactly.
Among learned-length samples, 95\% with exact joint allocations are judged correct, compared with 13\% for near-but-incorrect allocations and 8\% otherwise.
These descriptive associations help explain the large oracle-versus-learned accuracy gap in \Cref{tab:case_study_results_by_cell}, while highlighting the importance of accurate joint length allocation.

\subsection{Example Guidance and Failure Analysis}
\label{app:case_study_details:cases}
\label{app:case_study_details:results}
\label{app:software:failures}

Example guidance at \(\beta=0.5\) reduces pooled LLM-judged joint success by 2.9 percentage points with oracle lengths (\([-5.4,-0.6]\)) and 1.3 points with learned lengths (\([-3.8,+1.6]\)); \Cref{tab:case_study_results_by_cell} reports the configuration-specific results.
The oracle-length decline also persists in the prespecified sensitivity population (\(-3.3\) points, \([-6.1,-1.1]\)) and among 30 commits outside the narrower LLaDA development subset (\(-3.1\) points, \([-6.7,-0.1]\)).

\begin{table*}[t]
\centering
\caption{Joint LLM-judged success (\%) by configuration, with \(K=16\) samples per editing task and length setting and learned lengths sampled from the Java predictor.
Differences are in percentage points, with unadjusted 95\% paired commit-bootstrap intervals.
The pooled row weights 50 commits equally and is not the mean of the four configuration rows.}
\label{tab:case_study_results_by_cell}
\footnotesize
\setlength{\tabcolsep}{3.2pt}
\ifdefined\arxivbuild\setlength{\tabcolsep}{2pt}\fi %
\begin{tabular}{lrrrlrrll}
\toprule
 & & \multicolumn{3}{c}{Oracle lengths} & \multicolumn{3}{c}{Learned lengths} & Learned\(-\)oracle \\
\cmidrule(lr){3-5}\cmidrule(lr){6-8}
\((D,J)\) & Commits/tasks & \(\beta=0\) & \(\beta=0.5\) & Difference & \(\beta=0\) & \(\beta=0.5\) & Difference & at \(\beta=0\) \\
\midrule
(1,1) & 50/52 & 76.1 & 72.9 & \(-\)3.2 [\(-\)5.9, \(-\)0.8] & 17.3 & 17.1 & \(-\)0.2 [\(-\)3.4, +3.9] & \(-\)58.9 [\(-\)68.1, \(-\)49.5] \\
(1,2) & 26/28 & 74.9 & 72.7 & \(-\)2.2 [\(-\)4.7, \(-\)0.2] & 9.4 & 4.7 & \(-\)4.6 [\(-\)8.9, \(-\)1.3] & \(-\)65.5 [\(-\)79.8, \(-\)50.5] \\
(2,1) & 26/28 & 92.7 & 90.4 & \(-\)2.2 [\(-\)6.9, +0.9] & 21.2 & 18.6 & \(-\)2.5 [\(-\)9.3, +5.9] & \(-\)71.5 [\(-\)82.6, \(-\)59.1] \\
(2,2) & 17/19 & 82.7 & 84.5 & +1.8 [0.0, +5.2] & 9.2 & 9.6 & +0.4 [\(-\)5.2, +8.0] & \(-\)73.5 [\(-\)88.4, \(-\)56.6] \\
\midrule
Pooled & 50/127 & 74.4 & 71.5 & \(-\)2.9 [\(-\)5.4, \(-\)0.6] & 15.8 & 14.5 & \(-\)1.3 [\(-\)3.8, +1.6] & \(-\)58.6 [\(-\)67.4, \(-\)49.6] \\
\bottomrule
\end{tabular}
\end{table*}

Selection helps only when a sample pool contains both correct and incorrect rewrites.
Among 127 oracle-length pools, 87 contain only judged-correct samples, 17 only incorrect samples, and 23 both; the corresponding counts with learned lengths are 0, 55, and 72.
Selecting a correct rewrite whenever one is available could therefore improve pooled success by at most 8.3 points with oracle lengths, versus 45.2 points with learned lengths.
Even within the 23 mixed oracle-length pools, guidance lowers expected success from 62.0\% to 50.5\% (\(-11.5\) points, \([-19.7,-3.1]\)); this is a post-hoc comparison on an outcome-selected subset.

A post-hoc Checkstyle example (commit \texttt{b142dda472}, distinct from \Cref{fig:case_study_checkstyle_canvas}) illustrates how the reward can favor an incorrect rewrite.
The developer-consistent rewrite \(u_+\) of \texttt{logMessage} preserves the logged line number and \texttt{MSG\_DUPLICATE\_REGEXP}, whereas an incorrect rewrite \(u_-\) changes both, yet receives a higher reward (8.69 versus 2.05 nats).
The following code excerpts omit the enclosing method's closing braces.

\begin{codepanel}[colbacktitle=green!60!black!12]{Developer's rewrite \(u_+\)\hfill{\normalfont\footnotesize 7 of 16 draws}}
\begin{CaseSideCode}
final String msg = getMessage();
if (illegalPattern) {
    log(lineNumber, MSG_ILLEGAL_REGEXP, msg);
}
else {
    log(lineNumber, MSG_DUPLICATE_REGEXP, msg);
\end{CaseSideCode}
\end{codepanel}

\begin{codepanel}[colbacktitle=red!60!black!12]{Incorrect rewrite \(u_-\)\hfill{\normalfont\footnotesize 1 of 16 draws}}
\begin{CaseSideCode}
final String msg = getMessage();
if (illegalPattern) {
    log(lineNumber - 1, MSG_ILLEGAL_REGEXP, msg);
}
else {
    log(lineNumber, MSG_REQUIRED_REGEXP, msg);
\end{CaseSideCode}
\end{codepanel}

At \(\beta=0.5\), the reward difference gives each occurrence of \(u_-\) approximately 28 times the resampling weight of an occurrence of \(u_+\).
The single incorrect draw consequently receives 61.8\% of the total weight, compared with 15.7\% across the seven correct draws, lowering expected LLM-judged success from 43.8\% to 15.7\%.
Nevertheless, the ideal guided distribution favors \(u_+\) over \(u_-\) in their pairwise comparison:
\begin{equation*}
\frac{Q_{0.5}(u_-)}{Q_{0.5}(u_+)}
    = \exp(-3.597+3.319) \approx 0.76,
    \qquad
\frac{\widehat Q_{0.5,16}(u_-)}{\widehat Q_{0.5,16}(u_+)}
    = \frac{\exp(3.319)}{7} \approx 3.95.
\end{equation*}
The original sampling distribution favors \(u_+\) by a much larger margin than the observed seven-to-one draw count; finite-pool resampling therefore reverses this pairwise preference.
These ratios compare only the two displayed rewrites, not the total probability of a correct edit.

\section{Additional Discussion}
\label{app:additional_discussion}

\paragraph{Broader impacts.}
Multi-target few-shot prompting could enable example-guided document and code editing across multiple sites.
Practical use will require more reliable adherence to demonstrations and better prediction of output lengths.
With these improvements, users could specify transformations through a few examples rather than editing each site manually.

\paragraph{Potential risks.}
Incorrect transformations can propagate across sites, altering document meaning or introducing software defects.
Our Checkstyle case study shows that example guidance can favor an incorrect edit even when demonstrations support the intended transformation.
Practical use requires review, including compilation and testing for code; our evaluation instead relies on an LLM judge without execution.

\paragraph{Artifacts and licensing.}
The models are released under MIT (LLaDA~1.5), Apache~2.0 (dQwen3.5 and Qwen3), and the NVIDIA Open Model License (Nemotron-Labs-Diffusion).
The BIG-bench and SALT repositories use Apache~2.0 and MIT, respectively; SWE-Refactor is marked CC~BY~4.0, but its underlying Java files retain their upstream licenses.
We could not verify uniform redistribution terms for SST-2 or MIT Movie, so we do not assume unrestricted redistribution rights.

\paragraph{Compute.}
All experiments use 8--9B-parameter models on NVIDIA H100 NVL (96 GB) GPUs, one model per GPU; only the length predictors are trained.
Run logs suggest fewer than 100 H100 GPU-hours, although cumulative GPU time was not tracked separately.

\paragraph{AI usage.}
We used generative AI assistants for literature exploration, manuscript drafting and revision, and code and figure preparation.
The authors remain responsible for the methods, citations, results, and final text.
Separately, Claude Sonnet~5.5 was used to judge software edits under the experimental protocol of \Cref{sec:case_studies:experiments}.

\end{document}